%% file: main.tex
\documentclass{article}

\usepackage{PRIMEarxiv}

\usepackage[utf8]{inputenc} 
\usepackage[T1]{fontenc}    
\usepackage{hyperref}       
\usepackage{url}            
\usepackage{booktabs}       
\usepackage{amsfonts}       
\usepackage{nicefrac}       
\usepackage{microtype}      
\usepackage{lipsum}
\usepackage{fancyhdr}       
\usepackage{graphicx}       

\usepackage[numbers]{natbib}
\usepackage{amsmath}
\usepackage{xcolor}
\usepackage{amsthm}

\usepackage{algorithm}
\usepackage{algorithmic}
\usepackage{bbm}
\usepackage{adjustbox}
\usepackage{subcaption} 
\usepackage{multirow}
\usepackage{bm}
\newtheorem{theorem}{Theorem}

\newtheorem{lemma}{Lemma}
\newtheorem{prop}{Proposition}

\newtheorem{problem}{Problem}

\allowdisplaybreaks

\title{Safety Monitoring for Learning-Enabled Cyber-Physical Systems in Out-of-Distribution Scenarios}

\author{
    Vivian Lin \\
    University of Pennsylvania \\
    Philadelphia, Pennsylvania, USA \\
    \texttt{vilin@seas.upenn.edu} \\
\And
    Ramneet Kaur \\
    SRI \\
    Menlo Park, California, USA \\
    \texttt{ramneet.kaur@sri.com} \\
\And
    Yahan Yang \\
    University of Pennsylvania \\
    Philadelphia, Pennsylvania, USA \\
    \texttt{yangy96@seas.upenn.edu} \\
\And
    Souradeep Dutta \\
    University of British Columbia \\
    Vancouver, Canada \\
    \texttt{souradeep@ece.ubc.ca} \\
\And
    Yiannis Kantaros \\
    Washington University in St. Louis \\
    St. Louis, Missouri, USA \\
    \texttt{ioannisk@wustl.edu} \\
\And
    Anirban Roy \\
    SRI \\
    Menlo Park, California, USA \\
    \texttt{anirban.roy@sri.com} \\
\And
    Susmit Jha \\
    SRI \\
    Menlo Park, California, USA \\
    \texttt{susmit.jha@sri.com} \\
\And
    Oleg Sokolsky \\
    University of Pennsylvania \\
    Philadelphia, Pennsylvania, USA \\
    \texttt{sokolsky@cis.upenn.edu} \\
\And
    Insup Lee \\
    University of Pennsylvania \\
    Philadelphia, Pennsylvania, USA \\
    \texttt{lee@cis.upenn.edu}
}

\begin{document}

\maketitle

\begin{abstract}
The safety of learning-enabled cyber-physical systems is compromised by the well-known vulnerabilities of deep neural networks to out-of-distribution (OOD) inputs. Existing literature has sought to monitor the safety of such systems by detecting OOD data. However, such approaches have limited utility, as the presence of an OOD input does not necessarily imply the violation of a desired safety property. We instead propose to directly monitor safety in a manner that is itself robust to OOD data. To this end, we predict violations of signal temporal logic safety specifications based on predicted future trajectories. Our safety monitor additionally uses a novel combination of adaptive conformal prediction and incremental learning. The former obtains probabilistic prediction guarantees even on OOD data, and the latter prevents overly conservative predictions. We evaluate the efficacy of the proposed approach in two case studies on safety monitoring: 1) predicting collisions of an F1Tenth car with static obstacles, and 2) predicting collisions of a race car with multiple dynamic obstacles. We find that adaptive conformal prediction obtains theoretical guarantees where other uncertainty quantification methods fail to do so. Additionally, combining adaptive conformal prediction and incremental learning for safety monitoring achieves high recall and timeliness while reducing loss in precision. We achieve these results even in OOD settings and outperform alternative methods.
\end{abstract}

\input{sections/intro}
\input{sections/cartpole_motivating_example}
\input{sections/related}
\input{sections/background}
\input{sections/tech}
\input{sections/exp}

\input{sections/conc}

\section*{Acknowledgements}
This work was supported in part by ARO MURI W911NF-20-1-0080, NSF 2143274, NSF 2403758, NSF 2231257, and a gift from AWS AI to Penn Engineering's ASSET Center for Trustworthy AI. This work was also supported in part by the U.S. Air Force and DARPA under Contract No. FA8750-23-C-0519 and the U.S. Army Research Laboratory Cooperative Research Agreement W911NF-17-2-0196. Any opinions, findings, conclusions or recommendations expressed in this material are those of the authors and do not necessarily reflect the views the Army Research Office (ARO), the Department of Defense, or the United States Government.

\bibliographystyle{plainnat}
\bibliography{ref}

\appendix

\input{appendices/full_results}
\input{appendices/incremental_learning}

\input{appendices/fig_full_results}
\input{appendices/tbl_prediction_error}

\input{appendices/ncs_dist}
\input{appendices/fig_il_example}

\input{appendices/fig_ncs_dist}

\input{appendices/tbl_tvd}
\input{appendices/guarantees}
\input{appendices/tbl_empirical_coverage}

\input{appendices/repeatability}

\end{document}

%% file: sections/intro.tex
\section{Introduction}
\label{sec:intro} 

With the human-like performance of deep learning across different domains~\citep{galactica, atari}, there has been explosive interest in using such techniques for learning-enabled cyber-physical systems (LE-CPS). For instance, deep learning models have been deployed in autonomous vehicles for wide public use, such as in the Tesla Full Self-Driving system~\citep{tesla} and Waymo driverless taxis~\citep{waymo}. Despite their wide adoption, the critical vulnerability of deep learning models to out-of-distribution (OOD) inputs has yet to be fully understood or even corrected. That is, deep learning models have been shown to make mistakes on inputs that lay far from their training distribution, even those that are realistic and highly likely during deployment~\citep{hendrycks2016baseline,hendrycks2019benchmarking}.

One way to provide safety assurance for these systems is to detect scenarios in which OOD inputs are occurring, allowing the LE-CPS a chance to employ risk mitigation strategies (e.g., abstaining from making a prediction)~\citep{cai2020real,codit,ramakrishna2022efficient,interpretable_ood}. Although popular, such an approach has limited utility, as it assumes that out-of-distribution inputs are always destructive. However, contrary to this assumption, a) learning-enabled components might generalize to novel inputs to some extent~\citep{semantic_ood_detection}, and b) the overall system may be robust with respect to safety specifications despite component-level errors~\citep{falsification}. In Section~\ref{sec:cartpole}, we demonstrate this point empirically through an exploration of the cartpole control benchmark. When environmental changes influence the system, even a single-layer deep neural-network controller can generalize to the resulting OOD inputs and avoid failure.
Hence, OOD detection alone is not a complete solution to the problem of safety assurance.

In this paper, we alternatively propose to directly monitor the safety of an LE-CPS through methods that can be employed even in OOD scenarios. Making no assumptions on the input distribution, we predict the violation of a safety property over a finite horizon and raise an alarm accordingly. This safety property can be expressed as a signal temporal logic (STL) formula on the system's states~\citep{lars_iccps, falsification}, with degree of satisfaction captured by the STL robustness value. To monitor safety, we calculate the robustness value for a state trajectory predicted by a deep learning model. This approach is similar in spirit to some predictive runtime verification techniques~\citep{lars_iccps,zhao2024robust}, but we extend upon existing work to improve performance and obtain probabilistic guarantees even in OOD scenarios. Specifically, we present a technique that combines adaptive conformal prediction with incremental learning.

Adaptive conformal prediction (ACP) is an uncertainty quantification method that obtains probabilistic guarantees without any assumptions on the distribution of the predictor's inputs~\citep{acp}. This is in contrast to conformal prediction~\citep{saunders1999transduction,vovk1999machine}, which assumes exchangeability and hence no shift in distribution, and the robust conformal prediction that is used in~\citet{zhao2024robust}, which holds under a bounded amount of distribution shift. Using ACP alone would lead to more conservative safety violation predictions especially on OOD data, reducing precision for an increase in recall. In contrast to prior work that uses only uncertainty quantification for safety monitoring~\citep{zhao2024robust,lars_iccps}, we propose to use ACP with incremental learning to recover this precision.

We validate the proposed approach through two case studies of safety-critical systems, predicting collisions of an F1Tenth car with static obstacles and collisions of a race car with multiple dynamic obstacles. We compare our method to~\citet{zhao2024robust} and explore the empirical effects of incremental learning and uncertainty quantification on our safety monitor. 

The contributions of this paper can be summarized as follows.
\begin{enumerate}
\item We propose to monitor the safety of LE-CPS in OOD scenarios, instead of detecting 
    and abstaining on OOD inputs. Our method predicts safety based on a system's future STL robustness value.
\item We leverage the adaptive conformal prediction framework to obtain probabilistic guarantees on this prediction without any assumptions on the inference time data distribution.\
\item We employ incremental learning to balance the extra conservatism induced by adaptive conformal prediction.
\item We show empirically that, by combining adaptive conformal prediction with incremental learning, our proposed safety monitor predicts safety violations in a timely manner with competitive recall while balancing precision, even in OOD scenarios. Our monitor additionally makes predictions with probabilistic guarantees when alternative methods cannot. 
\end{enumerate}

%% file: sections/cartpole_motivating_example.tex
\section{A Motivating Example: Decoupling OOD Detection and Safety Monitoring} \label{sec:cartpole}
Although adopted widely the literature, using OOD detection to guard against safety violations unfairly assumes that OOD inputs to the learning-enabled CPS component always lead to system-level failures. In this section, we argue instead that OOD detection and safety monitoring are two distinct goals. We consider the cartpole (inverted pendulum) benchmark to motivate this decoupling.

\begin{figure*}[!t]
    \begin{subfigure}[h]{0.25\linewidth}
        \includegraphics[width=\linewidth]{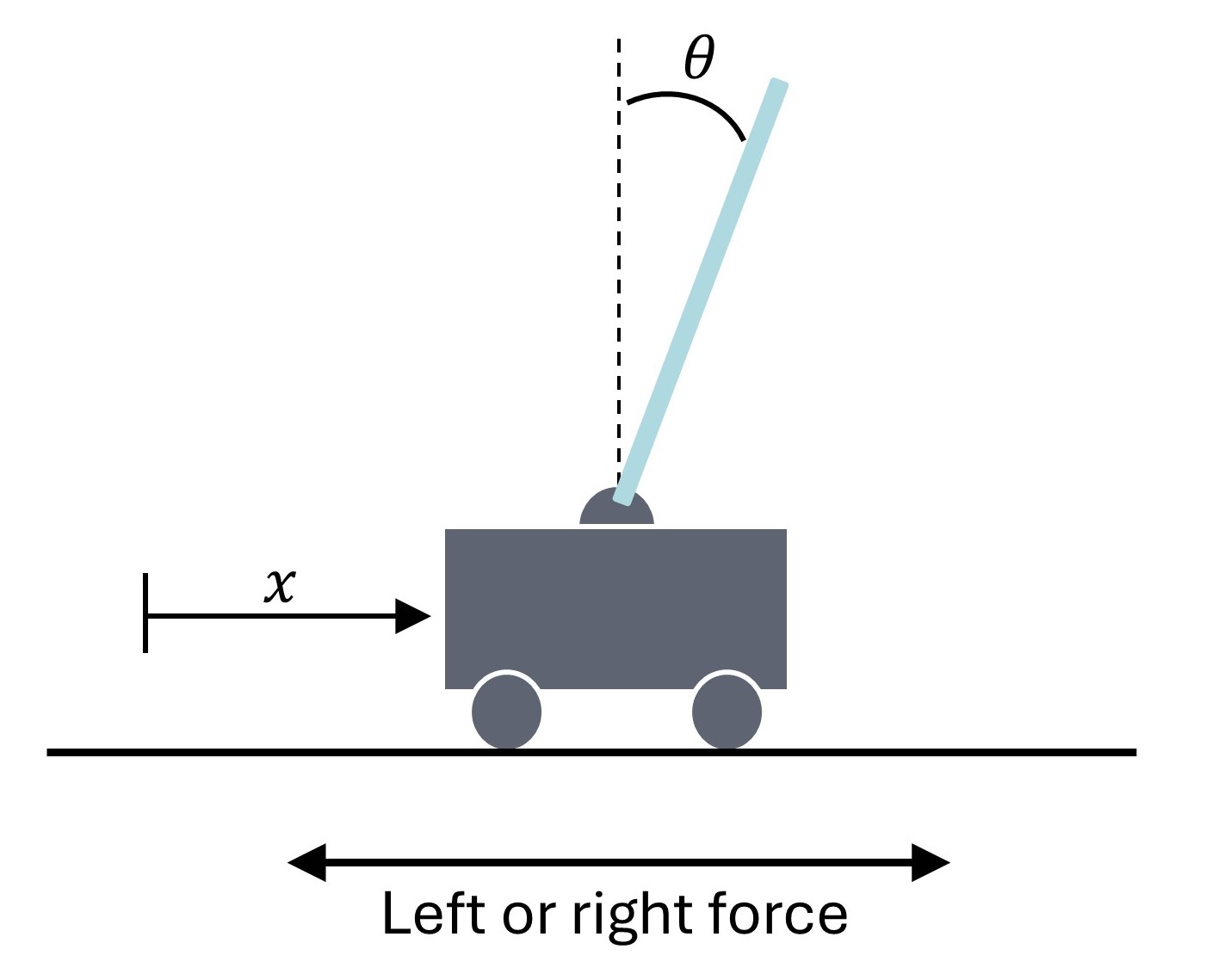}
        \caption{Cartpole Benchmark\label{fig:cartpole_diagram}}
    \end{subfigure}
    \begin{subfigure}[h]{0.7\linewidth}
        \includegraphics[width=\linewidth]{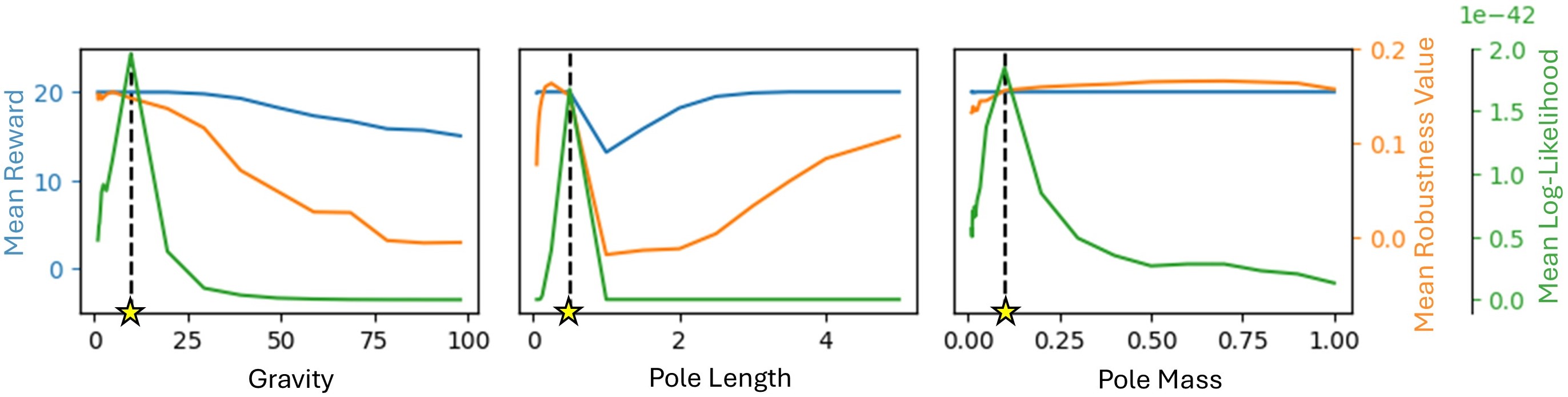}
        \caption{Safety Versus Distribution Shift\label{fig:cartpole_results}}
    \end{subfigure}%
    \caption{OOD inputs to the learning-enabled CPS component do not necessarily lead to safety violations. a) In the cartpole benchmark, a pole is attached to moving cart by a pivot point. The pole must be kept upright by applying left and right forces to the cart, while keeping the cart centered. b) We induce distribution shift in the cartpole's state trajectories by varying the environment parameters. The star indicates the in-distribution parameter selection. In this study, the reward measurement is a ground-truth evaluation of safety. The robustness value is a better indicator of the reward than the trajectory likelihood.}
\end{figure*}

In the cartpole benchmark, a pole is attached by a pivot point to a cart moving along an axis, as shown in Figure~\ref{fig:cartpole_diagram}. The controller must keep the pendulum upright by applying forces to the cart, while keeping the cart centered. The action space is a discrete scalar value, indicating a force pushing the cart to the left or the right. The observation space $s$ is a continuous four-dimensional vector, capturing the cart position $x$, cart velocity $\dot{x}$, pole angle $\theta$, and pole angular velocity $\dot{\theta}$,
$$
s = \begin{bmatrix}
    x & \dot{x} & \theta & \dot{\theta}
\end{bmatrix}^{\rm T}.
$$
A simulation episode is terminated as soon as the pole angle exceeds 12 degrees from the vertical axis or the cart position exceeds 2.4 units from the origin. The controller is rewarded for every time step when these conditions are met. We train a Deep Q Learning (DQN) agent with one hidden layer of 64 neurons to maximize this reward.

Following the reward function, we define a safe cartpole system as one that maintains a pole angle less than 12 degrees and a cart position less than 2.4 units from the origin for 20 time steps. The corresponding signal temporal logic (STL) specification and robustness value are
\begin{align*}
\varphi(s,t) &=  \square_{[0, 19]}\left( \left| s_{t,3} \right| < 12 \right) \wedge \left( \left| s_{t,0} \right| < 2.4\right),\\
\rho^\varphi(s,t) &= \min_{t'\in[t,t+19]} \min \left( 12 - \left|s_{t',3}\right|, \;2.4 - \left|s_{t',0}\right| \right),
\end{align*}
respectively. The flexibility of the STL framework allows many alternative properties to be monitored. For example, one may want that the cart position eventually reaches a goal location.

In our simulation experiments, our goal is to induce a distribution shift in the trajectories of the cartpole states and observe the resulting ability (or inability) of the system to satisfy the STL specification. To this end, we vary the gravity, pole length, and pole mass parameters of the cartpole environment, selecting specific values based on prior literature \cite{mohammed2021benchmark} and holding the remaining parameters constant. The resulting changes in system behavior lead to OOD inputs (i.e., states) to the neural network controller. Simulations are run from initial states to either termination or a maximum 20 time steps. Figure~\ref{fig:cartpole_results} shows the reward, STL robustness value, and trajectory log-likelihood calculated over each parameter selection and averaged over 100 simulations.

In this example, we use the reward to capture a ground-truth binary evaluation of the system's ability to satisfy the safety specifications at each time step. The measurement is incremented with each step when the pole angle and cart position are within the allowable ranges (with maximum value 20). According to the reward measurements, we find that the trained controller is robust on average to changes in the pole mass, but tends to perform worse for increasing gravity values and a specific range of the pole length.

The STL robustness value provides a more granular measurement of safety specification satisfaction. We expect this metric to track the reward function, as in this particular example, the two metrics capture the same exact performance metrics.

The log-likelihood of the 20-step simulation trajectories, calculated via Gaussian kernel density estimation, is a measure of similarity to the 20-step trajectories sampled from the training environment.\footnote{The low order of magnitude is an artifact from normalizing high-dimensional data.} This characterizes the distribution of state trajectories, peaking only for the in-distribution trajectories. Notably, for many parameter settings, the log likelihood drops sharply, while the robustness value remains high. Although OOD inputs to the cartpole controller occur, they do not necessarily lead to safety violations due to the controller's ability to generalize.

%% file: sections/related.tex
\section{Related Work}
\label{sec:related}
\textbf{OOD Detection.}
    Out-of-distribution (OOD) detection has been of significant research focus, particularly in exploiting the statistical or geometric differences between in-distribution and OOD data for standalone deep learning models~\citep{kaur, super_2, hendrycks2016baseline, unsuper_2, csi, aux, iDECODe, semantic}. OOD detection has also been explored in the cyber-physical systems (CPS) space via safety envelops for low-dimensional input sensors such as GPS \citep{tiwari2014safety}. More recently, there has been an increasing interest in detecting OOD and adversarial inputs in closed-loop CPS environments that employ high-dimensional sensors, such as cameras \citep{cai2020real, ramakrishna2022efficient, beta-vae-2, arvind, codit,k2021real, sridhar2022towards, time_series_adv, yang2024memory, kaur2024out}. Such methods allow the learning-enabled component to abstain from making likely erroneous predictions. Some of these approaches control false positives in detection either with conformal prediction~\cite{cai2020real, codit, interpretable_ood, kaur2024out} or with a human in the feedback loop~\cite{vishwakarma2024taming}. We argue that the presence of an OOD input does not necessarily imply the violation of a desired safety property for the system, and therefore propose to directly monitor safety properties in a manner that is robust (via adaptive conformal prediction) to OOD scenarios.

\textbf{Online Safety Monitoring.} The use of conformal prediction (CP) for runtime safety monitoring of LE-CPS with theoretical guarantees has been explored in the past~\citep{lars_iccps}. CP and therefore~\citet{lars_iccps}'s approach, however, assume that the inference time distribution of system inputs is the same as the training distribution. Recently, the use of robust conformal prediction (RCP) was proposed to overcome this limitation~\citep{zhao2024robust}. RCP's guarantees, however, still assume a bounded distance between the training and inference distributions. Assuming no distribution shift or bounded distribution shift at runtime~\citep{lars_iccps} may be unrealistic for trustworthy deployment of these systems in the real world. Our approach is built upon adaptive conformal prediction (ACP), which makes no assumptions on the runtime distribution. ACP has also been explored for motion planning with dynamic obstacle avoidance in the past~\cite{lars_adaptive}.

\textbf{Incremental Learning.} Incremental learning algorithms aim to continuously adapt a machine learning model to new classes or new distributions without catastrophically forgetting previously learned knowledge \citep{rebuffi2017icarl, he2020incremental, incremental_learning_image, ensemble-incremental}. For CPS, it is often desirable that the learning-enabled components adapt when the environment changes, and incremental learning is an efficient way to achieve this \citep{cps-incremental, reis2020unsupervised-incremental}. Previous approaches \citep{incremental_learning_image,yang2023incremental} have trained new leaf classifiers to handle the new distributions and reduce forgetting. Inspired by this line of work, we employ a set of predictors and corresponding distributions, which we dynamically select from at runtime. To the best of our knowledge, this work is the first to explore incremental learning with ACP for online safety monitoring with guarantees under any runtime distribution.

%% file: sections/background.tex
\section{Background}
\label{sec:bkgrnd}

In this section, we provide a basic but necessary overview of signal temporal logic and adaptive conformal prediction.
\input{sections/background_stl}
\input{sections/background_acp}

%% file: sections/background_stl.tex
\subsection{Signal Temporal Logic and Robustness Value}
Signal temporal logic (STL)~\citep{maler2004monitoring, stl} is a real-time temporal logic for specifying logical properties of signals. A signal is a function that maps time $t \in \mathcal{T}$ to the state $s \in \mathcal{S}$ of a continuous-time system. The syntax of an STL formula $\varphi$ is defined as
$$
  \varphi :=                            
    \mu                           \ | \
    \neg\varphi                   \ | \
    \varphi_1\wedge\varphi_2      \ | \
    \square_{[a,b]}\varphi        \ | \
    \Diamond_{[a,b]}\varphi        \ | \
    \varphi_1\mathcal{U}_{[a,b]}\varphi_2,
$$
where the signal predicate $\mu$ is a formula $f: \mathcal{S} \rightarrow \mathbb{R}$ with $f(s) > 0$ and $b>a\geq0$. The symbols $\wedge$, $\square$, $\Diamond$, and $\mathcal{U}$ denote the \textit{intersection}, \textit{always}, \textit{eventually}, and \textit{until} operators, respectively. A signal value $s$ satisfies ($\models$) an STL formula $\varphi$ at time $t$ under the following conditions:
\begin{align*}                         
    (s,t) \models \varphi                       &\quad \Leftrightarrow \quad    \mu(s(t))>0           
    \\
    (s,t) \models \neg\varphi               &\quad \Leftrightarrow \quad    \neg((s,t)\models \varphi)          
    \\
    (s,t) \models \varphi_1\wedge\varphi_2  &\quad \Leftrightarrow \quad    (s,t) \models \varphi_1\wedge (s,t) \models \varphi_2  
    \\
    (s,t) \models \square_{[a,b]}\varphi    &\quad \Leftrightarrow \quad    \forall t'\in[t+a,t+b], (s,t') \models \varphi 
    \\
    (s,t) \models \Diamond_{[a,b]}\varphi   &\quad \Leftrightarrow \quad    \exists t' \in [t+a, t+b], (s,t')\models \varphi    
    \\
    (s,t) \models \varphi_1\mathcal{U}_{[a,b]}\varphi_2     &\quad \Leftrightarrow \quad    \exists t' \in [t+a, t+b], (s,t')\models \varphi_2 \\ &\qquad\qquad \wedge \forall t'' \in [t,t'], (x,t'')\models \varphi_1.
\end{align*}

The above conditions produce a binary determination, indicating whether or not a signal at time $t$ satisfies the specified STL formula. For more granular evaluation, a robustness value $\rho^{\varphi}$ can be calculated to measure the degree of satisfaction (or violation). This can be applied to any STL formula, as follows:
\begin{align*}                         
    \rho^{\varphi}(s,t)                         &\quad \Leftrightarrow \quad    \mu(s(t))           
    \\
    \rho^{\neg \varphi}(s,t)                    &\quad \Leftrightarrow \quad   -\rho^{\varphi}(s,t)          
    \\
    \rho^{\varphi_1\wedge\varphi_2}(s,t)    &\quad \Leftrightarrow \quad    \min(\rho^{\varphi_1}(s,t), \rho^{\varphi_2}(s,t))  
    \\
    \rho^{\square_{[a,b]}\varphi}(s,t)      &\quad \Leftrightarrow \quad    \min_{t'\in[t+a,t+b]} \rho^{\varphi}(s,t')
    \\
    \rho^{\Diamond_{[a,b]}\varphi}(s,t)     &\quad \Leftrightarrow \quad    \max_{t'\in[t+a,t+b]} \rho^{\varphi}(s,t')   
    \\
    \rho^{\varphi_1\mathcal{U}_{[a,b]}\varphi_2}(s,t)       &\quad \Leftrightarrow \quad    \max_{t'\in[t+a, t+b]}(\min(\rho^{\varphi_2}(s,t'), \min_{t'' \in [t,t']}\rho^{\varphi_1}(s,t'')).
\end{align*}

A signal value $s$ satisfies an STL formula $\varphi$ at time $t$ if and only if the corresponding robustness value is positive:
$$
(s,t) \models \varphi \Leftrightarrow \rho^{\varphi}(s,t)>0.
$$

%% file: sections/background_acp.tex
\subsection{Adaptive Conformal Prediction} \label{sec:acp_background}
Conformal prediction~\citep{saunders1999transduction,vovk1999machine} is a statistical method, applicable to any predictive model $f:x\rightarrow y$, for obtaining prediction regions with a guaranteed probability of containing the correct label. Inductive conformal prediction (ICP)~\citep{papadopoulos2002inductive} is a variant of traditional conformal prediction that is commonly employed for its reduced computational burden. ICP requires a calibration set $D_{\text{cal}} = \{(x_i,y_i)\}_{i=1,2,\dots,n}$ held out from the training data. Given a non-conformity score (NCS) function, which measures the degree of similarity between new samples and the calibration data (e.g., prediction residual), a simple statistical analysis on the calibration set can generate prediction region $C(x_{n+1})$ for a new test point $x_{n+1}$ with unknown label $y_{n+1}$:
$$
\mathbb{P}\left( y_{n+1} \in C\left(x_{n+1}\right)\right) \geq 1-\delta,
$$
where $\delta\in(0,1)$ is the targeted coverage. For the remainder of this paper, we will sometimes refer to \textit{inductive conformal prediction} simply as \textit{conformal prediction}.

Crucially, ICP requires that $(x_1,y_1),\dots,(x_n,y_n), (x_{n+1},y_{n+1})$ be exchangeable. This assumption is violated in the cases of dependent (e.g., time-series) and out-of-distribution data. To address this issue,~\citet{acp} proposed the adaptive conformal prediction (ACP) framework, where the data generating distribution for new inputs can change from the underlying training distribution.

To achieve marginal $1-\delta$ coverage even with the data distribution shifting over time, ACP adaptively re-estimates a significance level $\delta_t$ at each time step $t$ and uses it to generate the prediction region based on the most recent observations: 
$$
\delta_{t+1} = \delta_{t} + \gamma(\delta-e_t).
$$
Here, $\gamma$ is the learning rate, and $e_t$ is the error at time $t$ estimated from the empirical miscoverage frequency of the current prediction region:
$$
e_t = \left\{\begin{matrix}
1, \text{ if  }Y_t \notin C_t,
\\ 
\;\;0, \text{ otherwise, }
\end{matrix}\right.
$$
where $C_t$ is the prediction region including all those predictions whose NCS lies in the $(1-\delta_t)^{\rm th}$ quantile of the calibration NCS set. The NCS set is updated online with the current observations.

While allowing for greater flexibility to OOD data, the ACP framework additionally obtains marginal coverage guarantees.
\begin{prop}[\citeauthor{acp}, \citeyear{acp}] With probability one we have that, for all $T \in \mathbb{N},$
    $$\left | \frac{1}{T} \sum_{t=1}^T e_t -\delta \right| \leq \frac{\max\{\delta_1, 1-\delta_1\}+\gamma}{T \gamma}.$$
\end{prop}
In particular, $\lim_{T\rightarrow \infty} \sum_{t=1}^T e_t = \delta$. This proposition states that ACP obtains the correct coverage frequency at the significance level $\delta$ over long intervals of time, irrespective of any assumptions on the data-generating distribution.

%% file: sections/tech.tex
\section{Problem Setting and Assumptions} \label{sec:setting}
Our goal is to monitor the safety of learning-enabled cyber-physical systems (LE-CPS). Specifically, we consider a class of discrete-time dynamical systems whose dynamics are not necessarily known. We assume that 1) there exists a desired safety property, 2) the safety property is static throughout monitoring, 3) the system states over time are available to the monitor, and 4) our safety monitor has error-free knowledge of the environment map, system states, and obstacles. This set of assumptions has important implications. First, the only required information about the LE-CPS is its system state. Second, the final assumption isolates safety monitoring from related but separate problems (e.g., system state estimation and object detection).

For the purposes of this work, we consider out-of-distribution inputs to a learning-enabled component as those sampled from a non-identical distribution to the component's training distribution. Out-of-distribution scenarios are those that generate out-of-distribution inputs. In the context of this work, OOD data are by extension also not exchangeable with the calibration set required for conformal prediction methods. It should be noted that the notion of OOD is not well-defined in the machine learning community, and assumptions of in-distribution data can often be false~\cite{recht2019imagenet}. For this reason, it is even more important that our proposed method makes no assumption on the distribution of the system's states. In evaluations, we use log-likelihood and divergence metrics as proxies for measuring change in distribution.

\section{Safety Monitoring}
\label{sec:failire_pred_tech}

In this section, we present our problem statements and proposed approach for this task. Our ultimate goal is to predict a safety violation by the LE-CPS in the near future and raise an alarm accordingly. This can be achieved by predicting the robustness value of a specified signal temporal logic (STL) formula on the system's future states.

\begin{problem}[Safety Monitoring]\label{prob:safety_mon}
    Consider a learning-enabled cyber-physical system with state $s$, operating over time period $t\in[0,T)$. Given an STL safety specification $\varphi$, predict the robustness score $\hat{\rho}^{\varphi}(s,t)$ at each time step $t$.
\end{problem}

The efficacy of any solution to Problem~\ref{prob:safety_mon} relies on the accuracy of the robustness score prediction. To provide defense against inaccurate predictions, we also consider how incorporating prediction regions may aid our safety monitor's performance.

\begin{problem}[Safety Monitoring with Probabilistic Guarantees]\label{prob:safety_mon_w_guarantees}
    Consider a learning-enabled cyber-physical system with state $s$, operating over time period $t\in[0,T)$. Given an STL safety specification $\varphi$ and a confidence level $\delta$, compute a prediction region $C$ on the robustness score $\hat{\rho}^{\varphi}(s,t)$ at each time step $t$ such that the average coverage probability of this prediction region converges to $1-\delta$. That is,
    $$
    1-\delta-p_1 \leq \frac{1}{T}\sum^{T-1}_{t=0} {\rm Prob}\left[ \rho^\varphi(s,t) \in C\left(\hat{\rho}^\varphi(s,t)\right) \right] \leq 1-\delta + p_2,
    $$
    where $\,\lim_{T\rightarrow\infty} p_1= 0$ and $\,\lim_{T\rightarrow\infty} p_2= 0$.
\end{problem}

In these problem formulations, our safety monitor only requires knowledge of the LE-CPS state, meaning that it is applicable to any LE-CPS. Furthermore, we reiterate that we make no assumption on the distribution of $s$, allowing for flexibility to any LE-CPS in any in- or out-of-distribution scenario.

We motivate our proposed solution to these problems with three key statements.
\begin{enumerate}
    \item \textit{Adaptive conformal prediction produces prediction regions with probabilistic guarantees under no assumption on the data distribution.}
    \item \textit{Adaptive conformal prediction enlarges prediction regions to compensate for prediction error.}
    \item \textit{Incremental learning reduces prediction error on novel inputs.}
\end{enumerate}
From the above, we hypothesize that while adaptive conformal prediction can achieve the guarantees desired in Problem~\ref{prob:safety_mon_w_guarantees}, it may lead to overly conservative safety violation predictions, especially for OOD data. However, by using the technique in concert with incremental learning, this conservatism can be limited.

Figure~\ref{fig:block_diagram} summarizes our proposed safety monitoring approach, which leverages both adaptive conformal prediction and incremental learning. Given an $h$-step history of the system's states, a trajectory predictor model predicts an $H$-step horizon of the system's future states. Based on this prediction, a robustness value is computed following the STL framework. Adaptive conformal prediction (ACP) then obtains a prediction region on this robustness value. An alarm can be raised if a safety violation is detected using this prediction region. Once we have the ground truth available, we can calculate error in the predicted robustness value. If this error exceeds a threshold, the corresponding history and horizon pair is saved. The original trajectory predictor is fine-tuned on these saved samples, adding a new predictor to a predictor set. Using a method based on K-means clustering, a model is dynamically selected from this predictor set at runtime.

\begin{figure}[!t] 
    \centering
    \includegraphics[width=0.6\columnwidth]{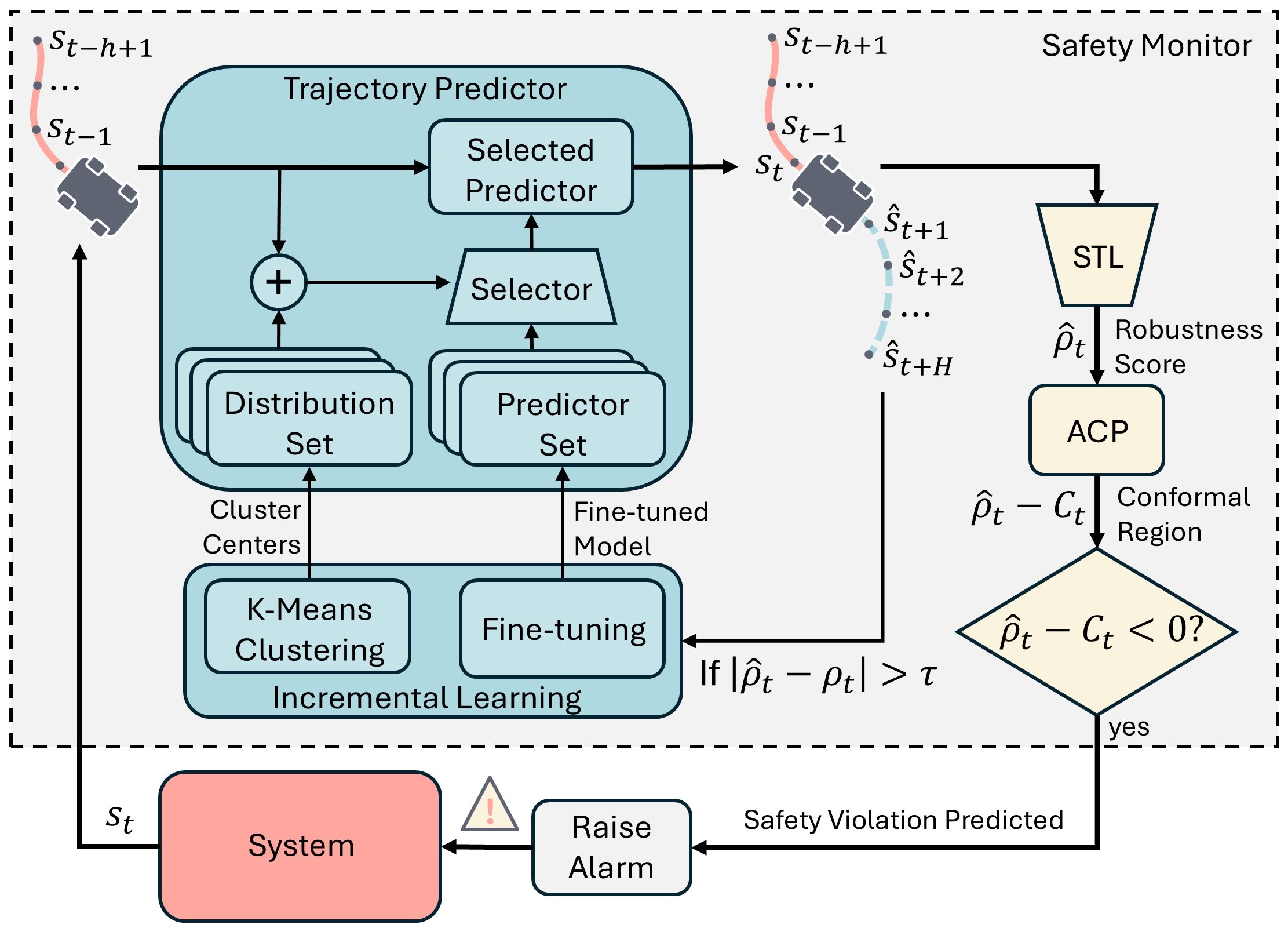}
    \caption{Safety monitoring for learning-enabled cyber-physical systems. Observing only the black-boxed system's states, we employ a trajectory predictor, updated via incremental learning, to predict the system's future states. On this prediction, we use the STL and ACP frameworks to obtain a prediction region on the robustness value. A simple condition indicates whether a violation has been predicted.
    \label{fig:block_diagram}}
\end{figure}

\subsection{Adaptive Conformal Prediction for Probabilistic Guarantees}
We require the adaptive conformal prediction (ACP) framework~\citep{acp} to provide prediction regions with probabilistic guarantee even when OOD inputs to the trajectory predictor occur. ACP predicts intervals in which the true robustness value is guaranteed to lie with high probability. 

\begin{lemma} Let $\gamma$ be a learning rate, $\delta \in (0,1)$ be the target failure probability, $t_0$ be the initial time, and $T$ be the total number of time 
steps. Let $C_t$ be the prediction region threshold obtained at time $t$ by adaptive conformal prediction. Then, for the robustness value prediction errors $\hat{\rho}_t^\varphi - \rho_t^\varphi$, it holds that:
    \begin{equation}        
        1-\delta-p_1 \leq \frac{1}{T} \sum_{t=t_0}^{T+t_0-1} {\rm Prob}(\hat{\rho}_{t}^{\varphi}-\rho_{t}^{\varphi}\leq C_{t}) \leq 1-\delta+p_2,
        \label{eq:lars_adaptive}
    \end{equation}
    with constants $p_1=\frac{\delta+\gamma}{T\gamma}$ and $p_2=\frac{(1-\delta)+\gamma}{T\gamma}$ that satisfy $\lim_{T\rightarrow \infty} p_1=0$ and $\lim_{T\rightarrow \infty} p_2=0$. 
    \label{lem:lars_adaptive}
\end{lemma}

\begin{proof}
    The proof follows from Corollary 3 of~\citet{lars_adaptive}, where the bound~\eqref{eq:lars_adaptive} is for the error ($1-e_{t}^1$) on one-step ahead state-value prediction: $||Y_{t}-\hat{Y}_{t-1}^{1}|| \leq C_{t}^{1}$. Here, prediction is on the robustness value of the future time-series window, and the error ($1-e_t$) on one-step ahead robustness-value prediction is captured by ${\hat{\rho}}^{\varphi}_{t}-{\rho}^{\varphi}_{t} \leq C_{t}$.
\end{proof}

Lemma~\ref{lem:lars_adaptive} states that the true robustness value $\rho_t^\varphi$ lies within the ACP prediction interval $[\hat{\rho}_t^\varphi-C_t, \infty)$ with probability approaching $1-\delta$ on average over time. We can, therefore, predict that no safety violation will occur if $\hat{\rho}_t^\varphi-C_t > 0$. Further, Lemma~\ref{lem:lars_adaptive} leads to the following theorem.

\begin{theorem} Let $\gamma$ be a learning rate, $\delta \in (0,1)$ be the target failure probability, $t_0$ be the initial time, and $T$ be the total number of time steps. Let $C_t$ be the prediction region threshold obtained at time $t$ by adaptive conformal prediction. If $\hat{\rho}_t^{\varphi} > C_{t} \;\forall t\in[t_0,T]$, then the probability of the system state $s$ satisfying the safety specification $\varphi$ at time $t$ is bounded below on average:
    \begin{equation}
        1-\delta-p_1 \leq \frac{1}{T} \sum_{t=t_0}^{T+t_0-1} {\rm Prob}((s,t)\models \varphi),
        \label{eq:satisfy_safety}
    \end{equation}
    with constant $p_1=\frac{\delta+\gamma}{T\gamma}$ that satisfies $\lim_{T\rightarrow \infty} p_1=0$. 
    \label{thm:satisfy_safety}
\end{theorem}
\begin{proof}
    If $\hat{\rho}^\varphi_t > C_t \; \forall t \in [t_0, T]$, then
    \begin{equation*}
        \hat{\rho}_t^{\varphi}-\rho_{t}^{\varphi}\leq C_{t} \implies
        \rho^\varphi_t > 0, \quad \forall t \in [t_0, T],
    \end{equation*}
    and
    \begin{equation*}
        {\rm Prob}\left(\hat{\rho}_t^{\varphi}-\rho_{t}^{\varphi}\leq C_{t}\right) \leq
        {\rm Prob}\left(\rho^\varphi_t > 0\right), \quad \forall t \in [t_0, T].
    \end{equation*}
    Since $\rho^\varphi_t > 0 \Leftrightarrow (s,t)\models \varphi$ for any $t$,  ~\eqref{eq:satisfy_safety} then follows from Lemma~\ref{lem:lars_adaptive}.
\end{proof}

Theorem~\ref{thm:satisfy_safety} provides a guarantee on the overall safety of the system if no violations are predicted. In other words, if $\hat{\rho}_t^{\varphi} \geq C_{t}$ for all $t \in [t_0, T]$, the probability of the system satisfying the safety specification $\varphi$ 
will be at least $(1-\delta-p_1)$ on average.

\subsection{Incremental Learning for Error Reduction on Out-of-Distribution Inputs}
Our technique makes no assumptions on the distribution of the system's states, thus allowing for settings where the trajectory predictor can make erroneous predictions on OOD inputs. We employ an incremental learning (IL) method, adapted from~\citet{yang2023incremental}, to guard against the hyper-conservatism that may result from using ACP on such data. We select the trajectory predictor at runtime from a distribution-predictor set $DP = \{(D_1,p_1),(D_2,p_2),\dots,(D_k,p_k)\}$, where each predictor $p_i$ in the set is trained on one (seen) distribution $D_i$ of the system's states. For trajectory predictions, we select the predictor corresponding to the distribution with the highest probability of generating the input state trajectory.

With $W$ as the set of high-error prediction trajectories collected at runtime, we now describe our IL approach on $W$. We use K-means clustering to generate prototypes of trajectories in $W$, and we consider those clusters which have a high ratio of samples from $W$ using a threshold on this ratio.\footnote{Prototypes can be generated using alternative methods, such as the \textit{memories} introduced by~\citet{yang2023incremental,interpretable_ood}.} Each cluster's center is then labeled with the distribution for which it is a prototype by estimating its distribution from the samples in the cluster.\footnote{For convenience, we overload the notation $D_i$ to indicate both the distribution and the cluster center obtained from samples of the distribution.} At inference time, we determine which cluster center provides the closest fit to the input data via clusters' distributions and select the corresponding predictor. By performing this selection dynamically, our IL technique mitigates the challenges of catastrophic forgetting that are common in continual update methods~\citep{mccloskey1989catastrophic}.

Algorithm~\ref{alg:inc_learning} shows our method for updating $DP$. The algorithm takes in a set $W$ of the trajectory history-horizon pairs for which the state predictor makes a prediction with error greater than some threshold $\tau$. This set can be collected at the inference time of Algorithm~\ref{alg:fail_pred}, which we will introduce later. Using K-means clustering, new cluster centers are generated (line~\ref{alg:line:kmeans}). For each new distribution, a new  predictor is trained on the trajectories in $W$ belonging to that distribution (line~\ref{alg:line:train}). In line~\ref{alg:line:append_dp}, the resulting distribution and predictor are appended to $DP$.

\begin{algorithm}[t]
    \caption{Incremental Learning for State Trajectory Predictors}
    \label{alg:inc_learning}
    \begin{algorithmic}[1]
       \STATE{\bfseries Input:} set $W$ of history-horizon pairs for fine-tuning, existing distribution-predictor set $DP=\{(D_1,p_1), \ldots, (D_k,p_k)\}$  of distributions $D_i$ and predictors $p_i$
       \STATE{\bfseries Output:} new distribution-predictor set $DP'=\{(D_1,p_1),$ $\ldots,(D_k,p_k), (D_{k+1},p_{k+1}), \ldots\}$
       \STATE Generate new K-means cluster centers $D_{k+1}, D_{k+2},\ldots$ from $W$\label{alg:line:kmeans}
       \FOR{$D_i \in \{D_{k+1}, D_{k+2},\ldots\}$}
            \STATE Train a predictor $p_i$ on $\{w \in W\ : w \sim D_i \}$ \label{alg:line:train}
            \STATE Append $(D_i, p_i)$ to $DP$ \label{alg:line:append_dp}
       \ENDFOR
       \STATE Return $DP$
    \end{algorithmic}
\end{algorithm}

\subsection{Safety Monitoring Algorithm}
Algorithm~\ref{alg:fail_pred} presents the proposed method for predicting an LE-CPS's safety violation. For a given STL safety specification $\varphi$, using a history of the system's past $h$ states, the algorithm raises an alarm when the system is predicted to violate $\varphi$ in the next $H$ time steps.

\begin{algorithm}[t]
    \caption{Safety Monitor for LE-CPS}
    \label{alg:fail_pred}
    \begin{algorithmic}[1]
        \STATE{\bfseries Input:} safety specification $\varphi$, distribution-predictor set $DP=\{(D_1,p_1),(D_2,p_2), \ldots \}$
        \STATE{\bfseries Parameter:} target failure probability $\delta\in(0,1)$, learning rate $\gamma$,  state history length $h$, prediction horizon $H$, start time $t_0>h+H$, threshold $\tau$ on the prediction error
        \STATE{\bfseries Initialize:} $\delta_{t} \gets \delta$, $R \gets \{\}$, $W \gets \{\}$
        
        \FOR{$t$ from $h$ to $\infty$} \label{alg:line:start}
            \STATE Select $p_i \in DP$ s.t. $[s_{t-h+1},\ldots,s_t]^{\rm T} \sim D_i$ \label{alg:line:dp}
            
            \STATE $\hat{s}_{t+1},\ldots,\hat{s}_{t+H}$ = $p_i(s_{t-h+1},\ldots,s_t)$ \label{alg:line:predict}
            
            \STATE $\hat{\rho}^{\varphi}_t = \hat{\rho}^{\varphi}(\hat{s}, t+1)$ \label{alg:line:rhohat}

            \IF{$t>h+H$}
            \STATE $\rho^{\varphi}_{t-H} = \rho^{\varphi}(s, t-H+1)$ \label{alg:line:rho}

            \STATE $R_{t} = \hat{\rho}^{\varphi}_{t-H}-\rho^{\varphi}_{t-H}$

            \STATE Append $R_t$ to the NCS set $R$ \label{alg:line:append_ncs} \hspace*{4.5em}%
            \rlap{\raisebox{\dimexpr.5\normalbaselineskip+.5\jot}{\smash{$\left.\begin{array}{@{}c@{}}\\{}\end{array}\color{red}\right\}%
          \color{red}\begin{tabular}{l}For ACP\end{tabular}$}}}
            \ENDIF

            \IF{$t>t_0$} 
            
            \STATE $C_{t}=\lceil (t)(1-\delta_{t}) \rceil^{\rm th}$ smallest $R_{i} \in R$ \label{alg:line:ct}
            
            \STATE$e_t = 0$ \text{if} $R_{t} \leq C_{t}$, 1 \text{o.t.w.} \hspace*{5.3em}%
            \rlap{\smash{$\left.\begin{array}{@{}c@{}}\\{}\\{}\end{array}\color{red}\right\}%
          \color{red}\begin{tabular}{l}For ACP\end{tabular}$}}
            
            \STATE $\delta_{t+1} = \delta_{t} + \gamma(\delta-e_t)$ \label{alg:line:deltat}

            \IF {$\left| R_t \right| > \tau$}
            \STATE Append $[s_{t-h-H+1},\ldots,s_t]^{\rm T}$ to $W$ \hspace*{0.8em}%
            \rlap{\smash{$\left.\begin{array}{@{}c@{}}\\{}\\{}\end{array}\color{red}\right\}%
          \color{red}\begin{tabular}{l}For IL\end{tabular}$}}
            \ENDIF
            
            \IF{$\hat{\rho}^{\varphi}_t < C_{t}$} \label{alg:line:compare}
            \STATE \text{Raise Alarm}
            \ENDIF
            
            \ENDIF
        \ENDFOR
    \end{algorithmic}
\end{algorithm}

The description of the algorithm is as follows. In lines~\ref{alg:line:dp} and~\ref{alg:line:predict}, the system's states over the next $H$ steps is predicted by the state trajectory predictor, selected from the distribution-predictor set $DP$. From the predicted states, the robustness $\hat{\rho}_t^{\varphi}$ of the system with respect to $\varphi$ at time $t$ for the next $H$ time steps is computed in line~\ref{alg:line:rhohat}. We raise an alarm if $\hat{\rho}_t^{\varphi} < C_t$ in lines 20 and 21.

$C_t$ is updated at runtime based on the adaptive conformal calibration set of non-conformity scores (NCS), which is gathered from recent observations. This is done as follows. Starting from time $t > h + H$, NCS are calculated in lines~\ref{alg:line:rho} to~\ref{alg:line:append_ncs}. This score $R_t$ at time $t$ is the difference between the predicted robustness $\hat{\rho}^\varphi_{t-H}$, and the actual robustness $\rho^\varphi_{t-H}$ of the system for the time interval $[t-H+1, t]$. Following~\citet{lars_adaptive}, we choose a time-lagged NCS, as only the observed states in the recent past $[t-H+1, t]$ are accessible at the current time step $t$ for computing the actual (or ground-truth) robustness $\rho^\varphi_{t-H}$ for the safety property $\varphi$ defined on the $H$-step window. Safety violation predictions, therefore, starts from the time $t > h + H$, to allow for sufficient NCS to be collected for initializing ACP. Lines~\ref{alg:line:ct} to~\ref{alg:line:deltat} perform the steps required to calculate the prediction region threshold $C_t$ based on the seen NCS. This threshold is calculated as the $(1-\delta_t)^{\rm th}$ quantile of the empirical distribution of the NCS set, where $\delta_t$ is adaptively updated at each step based on the learning rate $\gamma$ and the coverage error $e_t$ of the latest prediction region at time $t$.

In lines 17 to 19, data is gathered for incremental learning. If the error of the robustness score prediction exceeds a threshold, then the most recent $h+H$ states are saved for fine-tuning later.

%% file: sections/exp.tex
\section{Experiments}
\label{sec:failire_pred_exp}
In this section, we describe two case studies in which we evaluate our safety monitoring technique, our implementation details, and the results. Details for repeating our experiments can be found in Appendix~\ref{app:repeatability}. For each case study, we empirically confirm Lemma~\ref{lem:lars_adaptive} and Theorem~\ref{thm:satisfy_safety}. We additionally evaluate the empirical effect of uncertainty quantification methods and incremental learning on the safety monitoring task. We evaluate the following uncertainty quantification methods: point prediction, conformal prediction, robust conformal prediction, and adaptive conformal prediction.

Through our evaluation of robust conformal prediction (RCP), we compare our proposed method to direct robust predictive runtime verification~\citep{zhao2024robust}, the prior work most similar to our method (to our knowledge). \citet{zhao2024robust} use RCP to monitor for STL-encoded safety violations, under \textit{permissible} distribution shift. Crucially, RCP assumes that the distribution shift is within an $\epsilon$-bounded statistical distance from the calibration distribution, as measured by the f-divergence. Furthermore, for any desired confidence level $\delta$, this distance $\epsilon$ must satisfy $\epsilon < \delta$ and demands a minimum size offline calibration set that grows larger with $\epsilon$. Assuming these requirements are met, RCP guarantees that
$$
{\rm Prob}\left( \hat{\rho}_t^\varphi - \rho_t^\varphi \leq C_t \right) \geq 1-\delta.
$$
While RCP provides a slightly stronger guarantee than ACP (see Lemma~\ref{lem:lars_adaptive}), the latter can accommodate any distribution shift. Aside from the use of RCP, \citet{zhao2024robust} differs from our approach most significantly in that they do not use incremental learning.

We evaluate by three metrics, precision, recall, and timeliness. The timeliness metric is the number of time steps that an alarm is raised in advance of a safety property violation. In both case studies, the maximum timeliness is our prediction horizon of five steps.

\input{sections/f110}
\input{sections/racetrack}

\subsection{Safety Property}
For both case studies, we encode collision avoidance as our desired safety property and monitor for future violations of this property. Assume that the number and locations of the obstacles are known by a set of $w$ coordinate points, $\{(x^o_1,y^o_1),(x^o_2,y^o_2),\dots,(x^o_w,y^o_w)\}$. In the F1Tenth case study, these points denote the locations of the eight walls (i.e., obstacles). In the race car study, these points are the locations of the surrounding vehicles, which we assume are measured without error. The ego vehicle (i.e., F1Tenth car or race car) must maintain some minimum distance $c$ from each obstacle:
$$
\varphi(s,t) = \square_{[0,H]} \bigwedge_{i=1}^w  d\left(\left(s_{t,0},s_{t,1}\right),\left(x^o_i,y^o_i\right)\right) > c,
$$
where $d$ is the euclidean distance and $H > 0$. The corresponding robustness score is
$$
\rho^\varphi(s,t) = \min_{t'\in[t,t+H]} \min_{i\in[1,w]} d\left(\left(s_{t',0},s_{t',1}\right),\left(x^o_i,y^o_i\right)\right) - c.
$$
We select a horizon of $H=5$. For the F1Tenth case study, we select a safety threshold of $c=0.3$ meters. For the race car, we choose $c = 5.4$ meters. These safety thresholds were chosen to match those used to encode collision avoidance in the RL reward functions. However, any STL-encoded property can be monitored using our proposed approach. For example, using the \textit{until} operator, one can monitor that the ego vehicle maintains a slow speed until it is sufficiently far from an obstacle.

\subsection{Implementation}
To construct our datasets, we allocate 65\%, 15\%, and 20\% of our simulated in-distribution trajectories into a train, validation, and test set, respectively. For our trajectory predictors, we select a history and prediction horizon length of $5$ steps each. We obtain an offline calibration set for CP and RCP by sampling a single history and horizon pair randomly from each validation trajectory, ensuring independence~\citep{codit}.

For our F1Tenth Car case study, we train a feedforward neural network with two layers of 200 neurons each with a mean absolute error (MAE) loss. We train for 55 epochs with learning rate $5\times10^{-4}$. For fine-tuning, we use a weighted MAE loss $\mathcal{L} = \beta\mathcal{L}_S + (1-\beta)\mathcal{L}_C$, where $\mathcal{L}_C$ and $\mathcal{L}_S$ denote the MAE calculated over traces during which a crash does and does not occur, respectively. For the OOD settings where there are 3 and 5 missing LIDAR rays, we fine-tune the predictor for IL with a learning rate of $5\times10^{-4}$ and $\beta$ of 0.2 over 3000 and 4000 epochs, respectively. For the OOD settings with noisy LIDAR rays, we fine-tune with a learning rate of $5\times10^{-5}$ and $\beta$ of 1.0 (because crashes are less common in these settings) over 2000 epochs. For all trainings, we apply a weight decay of $10^{-4}$.

For the Race Car case study, we use AgentFormer~\citep{agentformer}, a transformer variant that leverages the attention mechanism to model both the social and temporal aspects of the system. This allows AgentFormer to account for interactions among vehicles when predicting the future states of the system. We train with the AgentFormer loss for 20 epochs with a learning rate of $10^{-4}$. For fine-tuning, we train for 80 epochs with a learning rate of $10^{-5}$.

For collecting high-error traces for fine-tuning during incremental learning, we select error threshold $\tau$ as the 80\% and 50\% quantile of in-distribution errors for the F1Tenth Car and Race Car case studies, respectively. To obtain distribution prototypes for incremental learning, we select the number of clusters by the elbow method and the threshold for selecting clusters by a hyperparameter search.

Finally, for ACP, we choose $\delta=0.1$ and $\gamma=0.005$, following~\citet{acp}. For RCP, we select the largest choice of $\epsilon<\delta$ (with step size 0.01) that satisfies the calibration size requirements for RCP: $\epsilon=0.08$ for the F1Tenth Car and $\epsilon=0.03$ for the Race Car case study. We choose start time $t_0=15$ steps.

\subsection{Results}
\input{sections/results_table}
Tables~\ref{tab:recall},~\ref{tab:timeliness}, and~\ref{tab:precision} report respectively the recall, timeliness, and precision of our safety monitoring technique (ACP + IL) for both case studies. For each metric, we report mean and standard deviation over 10 trials. We also report these metrics for our baseline~\citep{zhao2024robust} (RCP) and variants where only point prediction (PP), conformal prediction (CP), and adaptive conformal prediction (ACP) are used. For brevity, we include for each case study only the in-distribution (ID) setting and the most severe of each type of OOD setting. The results for all OOD settings follow similar trends, as shown in Appendix~\ref{app:full_results}.

\textbf{Adaptive conformal prediction and incremental learning achieve competitive recall and timeliness.} Overall, our method either outperforms or performs comparably to alternatives in terms of recall and timeliness (see Tables~\ref{tab:recall} and~\ref{tab:timeliness}). In all settings, ACP alone achieves competitive performance compared to point, conformal, and robust conformal prediction. Additionally, ACP almost always maintains high recall and timeliness even for OOD data. In OOD cases where ACP alone does not recover high recall and timeliness (e.g., five missing rays), IL closes this gap. Furthermore, across the board, the addition of IL boosts these two metrics. In ablation studies (see Appendix~\ref{app:full_results}), we found that IL also improves the recall and timeliness of PP, CP, and RCP, although our method maintains strongest performance. We discuss our IL technique, including its evasion of catastrophic forgetting, in more detail in Appendix~\ref{app:forgetting}.

\textbf{Incremental learning improves precision.} Lowered precision is a natural consequence of uncertainty quantification, as the technique leads to more conservative predictions based on prediction regions. Indeed, ACP trades off precision for recall in many of our settings. While high recall is essential for safety-critical systems, a low precision is also undesirable, as it indicates excessive false alarms. Overall, we find that our use of IL with ACP recovers a degree of the lost precision (see Table~\ref{tab:precision}). Hence, the combination of IL with ACP is essential to balance the recall-precision trade-off. The only exception is when the distribution of the OOD non-conformity scores is thin-tailed with a mean close to that in the ID case (e.g., five missing rays). We study this limitation further in Appendix~\ref{app:ncs_dist}.

\textbf{Adaptive conformal prediction obtains theoretical guarantees.} Among the uncertainty quantification methods we evaluate, only adaptive conformal prediction obtains probabilistic guarantees for \textit{any} distribution shift. In Figure~\ref{fig:acp_example}, we show the empirical coverage rate and STL-satisfaction rate of ACP in both case studies, confirming Lemma~\ref{lem:lars_adaptive} and Theorem~\ref{thm:satisfy_safety}. As dictated by Lemma~\ref{lem:lars_adaptive}, the empirical coverage rate remains within the theoretical envelope, which narrows towards the targe $1-\delta$ coverage as time progresses. Additionally, the empirical satisfaction rate remains above the lower bound of Theorem~\ref{thm:satisfy_safety} and approaches 1.

\textbf{Conformal prediction and robust conformal prediction fail to obtain theoretical guarantees.} In contrast to ACP, the assumptions necessary to obtain guarantees with CP cannot be satisfied for the time-series and out-of-distribution data we consider. Similarly, although RCP can allow for guarantees, the distribution shift in our settings far exceeds the amount permitted by the technique. In Appendix~\ref{app:emp_cov}, we estimate the total variation distance between the calibration and inference non-conformity scores, following the method outlined in~\citet{zhao2024robust}. The estimated distances exceed our choices of $\epsilon$ in all settings, leading to empirical coverage below the target $1-\delta$. We emphasize that these distances even exceed $\delta=0.1$, making RCP impossible to apply unless we lower our desired coverage guarantee $1-\delta$.

\begin{figure}[!t] 
    \centering
    \includegraphics[width=0.7\columnwidth]{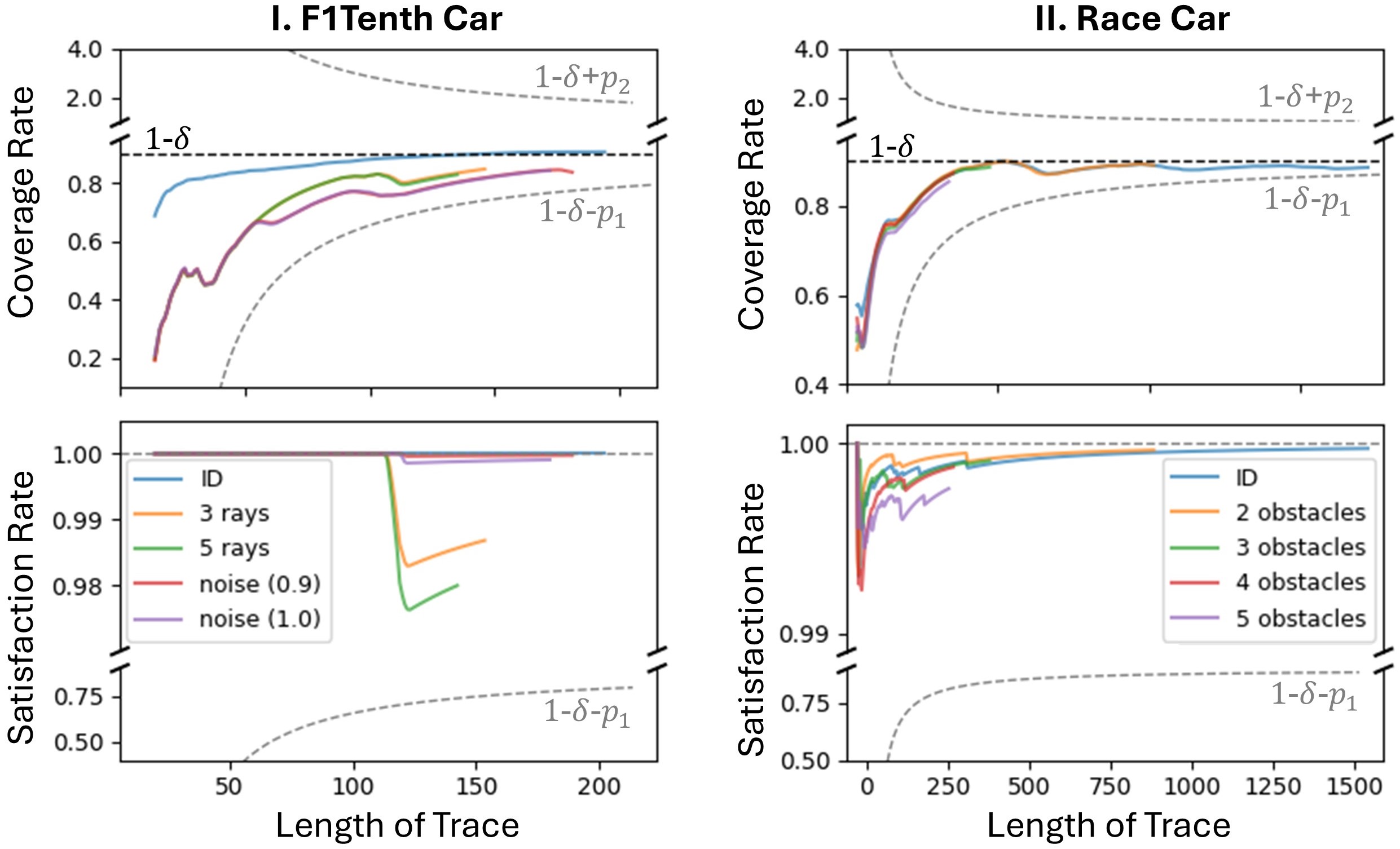}
    \caption{Empirical evaluations of Lemma~\ref{lem:lars_adaptive} and Theorem~\ref{thm:satisfy_safety} for both case studies (without IL). For more accurate estimates, values are calculated over the union of all 10 trials. Lemma~\ref{lem:lars_adaptive} (top): the empirical ACP coverage rates are within the theoretical bounds in the ID and OOD scenarios. Theorem~\ref{thm:satisfy_safety} (bottom): for ID and OOD simulations where the assumptions of Theorem~\ref{thm:satisfy_safety} hold, the empirical STL satisfaction rates are within the theoretical bounds. Sudden drops occur at times when the system reaches a region that is challenging to safely navigate (e.g., sharp corners).
    \label{fig:acp_example}}
\end{figure}

%% file: sections/f110.tex
\subsection{Case Study I: F1Tenth Car and Static Obstacles}

\begin{figure*}[!t]
    \centering
    \begin{subfigure}[h]{0.4\textwidth}
        \includegraphics[width=\linewidth]{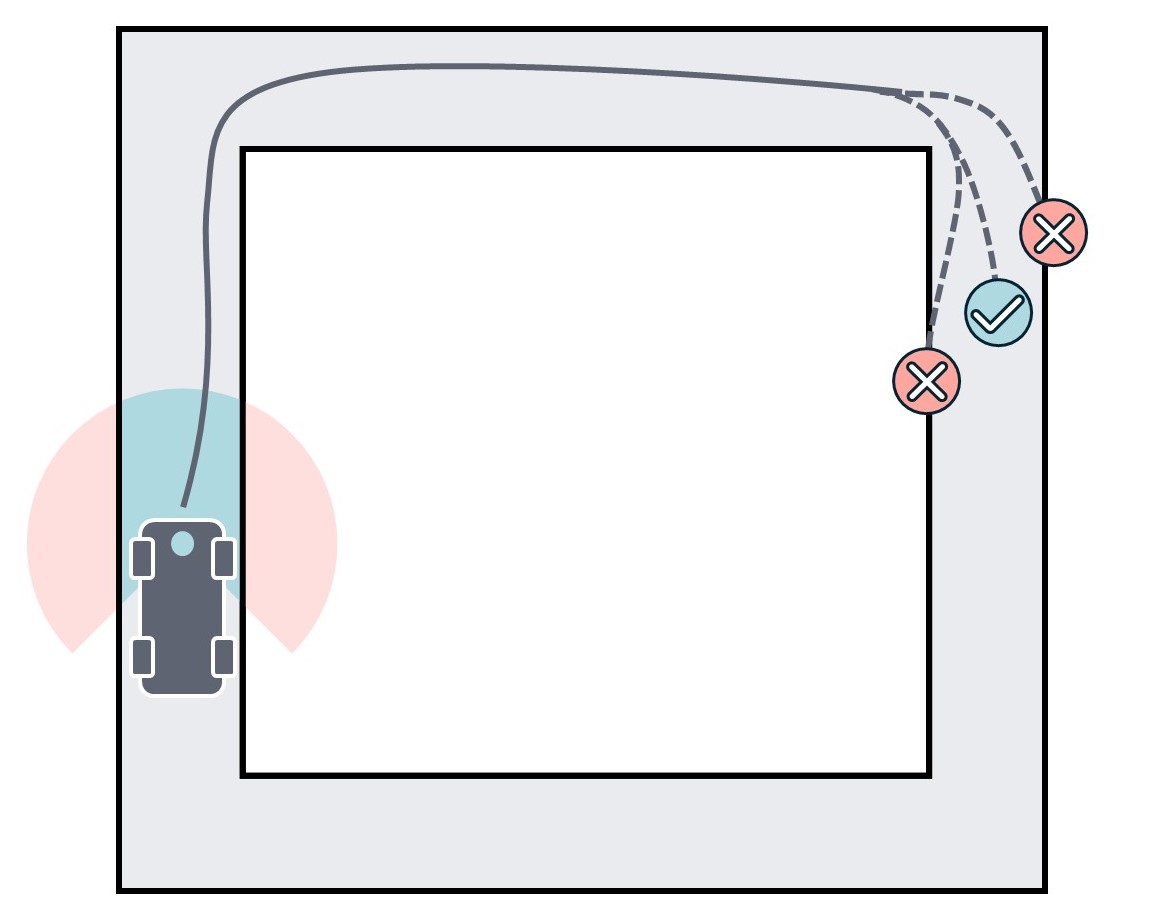}
        \caption{Case Study I. \label{fig:f110}}
    \end{subfigure}
    \hspace{0.1\textwidth}
    \begin{subfigure}[h]{0.4\textwidth}
        \includegraphics[width=\linewidth]{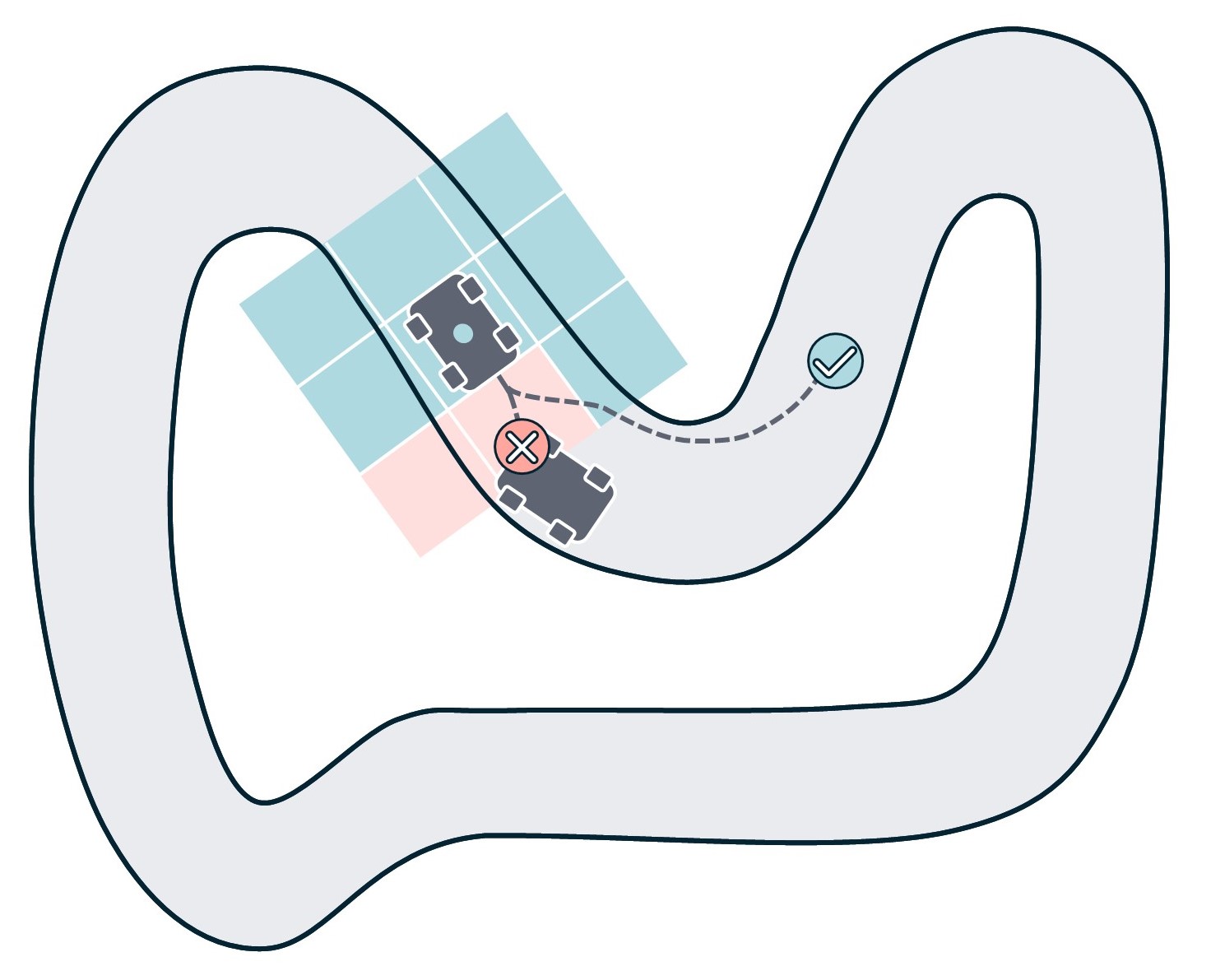}
        \caption{Case Study II. \label{fig:racetrack}}
    \end{subfigure}%
    \caption{Case studies for empirical evaluation. a) A deep RL agent drives an F1Tenth car through series of hallways. Based on LIDAR measurements, the agent must select actions that avoid collisions with the walls. Our safety monitor predicts collisions with the walls. b) A deep RL agent traverses a race track. Based on grids that provide the locations of nearby vehicles and the road surface (the former is shown in this figure), the agent must select actions that avoid collisions with surrounding vehicles and keep the race car on the track. Our safety monitor predicts collisions with other vehicles.}
\end{figure*}

In this case study from \citet{ivanov2020case}, a simulated F1Tenth car must navigate a hallway with four 90-degree right turns forming a square (see Figure~\ref{fig:f110}). The hallway is 1.5 meters in width, with each side 20 meters in length. To complete this task, the controller relies on feedback from 21 LIDAR rays, each with a 5 meter range and together covering $-$135 to 135 degrees relative to the car's heading. The controller determines a continuous space steering command. Constant throttle is assumed. The state of the F1Tenth car consists of its position $(x,y)$, linear velocity $v$, and heading $\theta$:
$$s = \begin{bmatrix}
    x & y & v & \theta
\end{bmatrix}^{\rm T}.$$

\citet{ivanov2020case} train a deep RL-based neural network controller that maps LIDAR measurements to steering commands. We use a deep deterministic policy gradient (DDPG) controller with two layers and 64 neurons per layer from this previous work. The controller was trained using a reward function that discourages collisions and promotes smooth control.

We simulated state trajectories from initial states to collision or maximum time step of 200. We collected 407 in-distribution samples by iterating over initial distance from the walls to the left and front of the car (ranging within $[0.3,1.2)$ in steps of 0.3 meters and $[0,20)$ in steps of 2 meters, respectively) and the heading (ranging within $[-0.45, 0.45)$ in steps of 0.05 radians). We consider four scenarios that lead to OOD LIDAR input to the F1Tenth car controller, impacting the car's safety: 1) three rays are missing from the LIDAR data, 2) five rays are missing from the LIDAR data, 3) uniform noise in $[0, 0.9]$ is added to the LIDAR data, and 4) uniform noise in $[0, 1.0]$ is added to the LIDAR data. For each OOD setting, we collected 200 samples with initial distance from the left and front walls of 0.75 and 9.0 meters, respectively, and intial heading of 0 radians. In both the in-distribution and OOD settings, the collected data excludes trajectories shorter than 25 steps.

%% file: sections/racetrack.tex
\subsection{Case Study II: Race Car and Dynamic Obstacles
}

Here, we consider a reinforcement-learning (RL) based race car environment by~\citet{highway-env}. A race car must drive on a two-lane race track populated with other vehicles (see Figure~\ref{fig:racetrack}). Each lane is five meters wide. The race car controller receives feedback in the form of an occupancy grid and an on-road grid, indicating the presence of a vehicle or the road surface in each grid location, respectively. The grid covers $-$18 to 18 meters in both Cartesian directions, with three meter step. The controller computes a steering command, with constant throttle assumed. The state of the race car consists of its position $(x,y)$, linear velocity $v$, and heading $\theta$:
$$s = \begin{bmatrix}
    x & y & v & \theta
\end{bmatrix}^{\rm T}.$$
We train a deep RL controller by the proximal policy optimization (PPO) algorithm for 10 epochs with a learning rate of $5\times 10^{-4}$. The actor and value networks are two layers each, with 256 neurons per layer. The controller maps grid observations to steering commands. To train the controller, we use a reward function that discourages collisions and promotes staying on the race track.

We simulated state trajectories from 100 randomized initial states to collision or maximum time step of 300 (1500 policy steps), and we exclude those shorter than 25 steps. We repeat this process for one in-distribution and four out-of-distribution scenarios. In our in-distribution (i.e., training) scenario, there is one vehicle on the road in addition to the ego vehicle. We also consider four scenarios that lead to out-of-distribution occupancy grid inputs to the race car controller: up to two, three, four, and five other vehicles on the road in addition to the ego vehicle. In total, we obtain 67 in-distribution samples, 85 two-vehicles samples, 99 three-vehicle samples, 99 four-vehicle samples, and 98 five-vehicle samples.

%% file: sections/results_table.tex
\begin{table}[t]
\caption{Recall for both case studies, recorded over 10 trials. Our safety monitor (ACP+IL) outperforms alternatives in nearly all cases. For ID race car simulations, ACP performs comparably to CP.
 \label{tab:recall}}
\centering
     \begin{tabular}{c|ccc|cc|}
     \cline{2-6}
     & \multicolumn{3}{c|}{\textbf{Case Study I}} & \multicolumn{2}{c|}{\textbf{Case Study II}} \\
     \hline
     \multicolumn{1}{|c|}{\textbf{Tech.}} & \textbf{ID} & \textbf{5 rays} & \textbf{noise (1.0)} & \textbf{ID} & \textbf{5 obs.}\\
     \hline
     \multicolumn{1}{|c|}{PP} & 0.40 $\pm$ 0.33 & 0.48 $\pm$ 0.06 & 0.79 $\pm$ 0.11 & 0.93 $\pm$ 0.04 & 0.82 $\pm$ 0.06 \\
     \hline
     \multicolumn{1}{|c|}{CP} & 0.66 $\pm$ 0.20 & 0.56 $\pm$ 0.05 & 0.90 $\pm$ 0.08 & \textbf{0.96 $\pm$ 0.03} & 0.90 $\pm$ 0.07 \\
     \hline
     \multicolumn{1}{|c|}{RCP} & 0.50 $\pm$ 0.45 & 0.45 $\pm$ 0.14 & 0.63 $\pm$ 0.25 & 0.86 $\pm$ 0.15 & 0.77 $\pm$ 0.12 \\
     \hline
     \multicolumn{1}{|c|}{ACP} & \textbf{0.90 $\pm$ 0.15} & 0.56 $\pm$ 0.05 & 0.96 $\pm$ 0.04 & \textbf{0.96 $\pm$ 0.03} & 0.97 $\pm$ 0.01 \\
     \hline
     \multicolumn{1}{|c|}{ACP+IL} & - & \textbf{0.94 $\pm$ 0.02} & \textbf{0.98 $\pm$ 0.02} & - & \textbf{0.98 $\pm$ 0.01} \\
     \hline
     \end{tabular}
\end{table}

\begin{table}[t]
\caption{Timeliness for both case studies, recorded over 10 trials. Our safety monitor (ACP+IL) outperforms alternatives in nearly all cases. For ID race car simulations, CP achieves higher timeliness than ACP by a narrow margin.
 \label{tab:timeliness}}
\centering
     \begin{tabular}{c|ccc|cc|}
     \cline{2-6}
     & \multicolumn{3}{c|}{\textbf{Case Study I}} & \multicolumn{2}{c|}{\textbf{Case Study II}} \\
     \hline
     \multicolumn{1}{|c|}{\textbf{Tech.}} & \textbf{ID} & \textbf{5 rays} & \textbf{noise (1.0)} & \textbf{ID} & \textbf{5 obs.}\\
     \hline
     \multicolumn{1}{|c|}{PP} & 2.75 $\pm$ 1.48 & 2.91 $\pm$ 0.22 & 4.02 $\pm$ 0.51 & 4.84 $\pm$ 0.11 & 4.16 $\pm$ 0.29 \\
     \hline
     \multicolumn{1}{|c|}{CP} & 3.30 $\pm$ 0.98 & 3.13 $\pm$ 0.19 & 4.49 $\pm$ 0.41 & \textbf{4.89 $\pm$ 0.12} & 4.52 $\pm$ 0.32 \\
     \hline
     \multicolumn{1}{|c|}{RCP} & 4.17 $\pm$ 1.18 & 2.83 $\pm$ 0.56 & 3.61 $\pm$ 0.87 & 4.53 $\pm$ 0.67 & 4.08 $\pm$ 0.44 \\
     \hline
     \multicolumn{1}{|c|}{ACP} & \textbf{4.50 $\pm$ 0.77} & 3.11 $\pm$ 0.18 & 4.80 $\pm$ 0.19 & 4.87 $\pm$ 0.11 & 4.88 $\pm$ 0.06 \\
     \hline
     \multicolumn{1}{|c|}{ACP+IL} & - & \textbf{4.75 $\pm$ 0.07} & \textbf{4.89 $\pm$ 0.09} & - & \textbf{4.93 $\pm$ 0.05} \\
     \hline
     \end{tabular}
\end{table}

\begin{table}[t!]
\caption{Precision for both case studies, recorded over 10 trials. IL reduces the cost in precision incurred by ACP.
 \label{tab:precision}}
\centering
     \begin{tabular}{c|ccc|cc|}
     \cline{2-6}
     & \multicolumn{3}{c|}{\textbf{Case Study I}} & \multicolumn{2}{c|}{\textbf{Case Study II}} \\
     \hline
     \multicolumn{1}{|c|}{\textbf{Tech.}} & \textbf{ID} & \textbf{5 rays} & \textbf{noise (1.0)} & \textbf{ID} & \textbf{5 obs.}\\
     \hline
     \multicolumn{1}{|c|}{PP} & \textbf{0.93 $\pm$ 0.12} & \textbf{1.00 $\pm$ 0.01} & \textbf{0.97 $\pm$ 0.04} & \textbf{0.82 $\pm$ 0.07} & \textbf{0.71 $\pm$ 0.04} \\
     \hline
     \multicolumn{1}{|c|}{CP} & 0.68 $\pm$ 0.26 & 0.85 $\pm$ 0.22 & 0.79 $\pm$ 0.16 & 0.74 $\pm$ 0.06 & 0.67 $\pm$ 0.05 \\
     \hline
     \multicolumn{1}{|c|}{RCP} & 0.72 $\pm$ 0.19 & 0.87 $\pm$ 0.29 & 0.87 $\pm$ 0.28 & \textbf{0.82 $\pm$ 0.09} & \textbf{0.71 $\pm$ 0.05} \\
     \hline
     \multicolumn{1}{|c|}{ACP} & 0.56 $\pm$ 0.12 & 0.92 $\pm$ 0.15 & 0.61 $\pm$ 0.13 & 0.75 $\pm$ 0.04 & 0.54 $\pm$ 0.05 \\
     \hline
     \multicolumn{1}{|c|}{ACP+IL} & - & 0.76 $\pm$ 0.10 & 0.79 $\pm$ 0.08 & - & 0.59 $\pm$ 0.06 \\
     \hline
     \end{tabular}
\end{table}

%% file: sections/conc.tex
\section{Conclusions} \label{sec:conc}
In this paper, we presented a method to monitor the safety of learning-enabled cyber-physical systems, which can be vulnerable to out-of-distribution scenarios. We demonstrated that this direct safety monitoring is desirable over the OOD detection approach used in many existing literature, as OOD inputs may not necessarily lead to safety violations. We employed a combination of adaptive conformal prediction and incremental learning to obtain probabilistic guarantees even on OOD system state trajectories, while limiting hyper-conservatism. Our empirical results demonstrated that combining these two methods is instrumental to our safety monitor. Adaptive conformal prediction obtains theoretical guarantees under \textit{any} amount of distribution shift, while conformal and robust conformal prediction cannot. Additionally, the use of both ACP and IL drastically increases the recall and timeliness of our method in OOD settings, while reducing the cost in precision.

A variety of extensions present interesting avenues for future work. In particular, the efficacy of our approach is closely linked to the fine-tuning of the trajectory predictor. Alternative methods for selecting fine-tuning data at runtime may allow larger and richer datasets to be collected, improving the precision of our safety monitor. Fine-tuning the trajectory predictor can also be computationally expensive. Exploring smaller-scale models may help to address this bottleneck. Finally, our safety monitor acts in online settings for learning-enabled systems trained offline, but the method is also applicable to those trained online. Evaluation in this latter setting may provide valuable insights into the resilience of our method to novel situations, as models trained online continually influence the distribution of the LE-CPS states. These extensions may further improve our safety monitor for learning-enabled cyber-physical systems in out-of-distributions scenarios.

%% file: appendices/full_results.tex
\section{Complete Results and Ablations} \label{app:full_results}
Figure~\ref{fig:both_results} reports the recall, precision, and timeliness of our safety monitor (ACP + IL) for both case studies in all OOD settings. We evaluate point prediction (PP), conformal prediction (CP), robust conformal prediction (RCP), and adaptive conformal prediction (ACP) with and without incremental learning (IL). 

%% file: appendices/incremental_learning.tex
\section{Incremental Learning} \label{app:forgetting}
Table~\ref{tab:pred_err} shows the average displacement error, defined in~\citet{agentformer}, of our predictors (Appendix~\ref{app:ncs_dist} discusses the $P(|R_t|>\tau)$ column). On OOD data, error increases. With incremental learning (IL), prediction performance is almost always recovered to in-distribution levels. This is largely because our IL technique mitigates the effects of catastrophic forgetting common in continual learning. For example, consider Case Study I with three missing LIDAR rays. Figure~\ref{fig:f110_3_il} shows an example trace. While the fine-tuned model improves predictions on challenging parts of the course, it "forgets" previously learned knowledge about the rest. A combination of the original and fine-tuned model must be used to compensate.

%% file: appendices/fig_full_results.tex
\begin{figure*}[!t] 
    \centering
    \includegraphics[width=\textwidth]{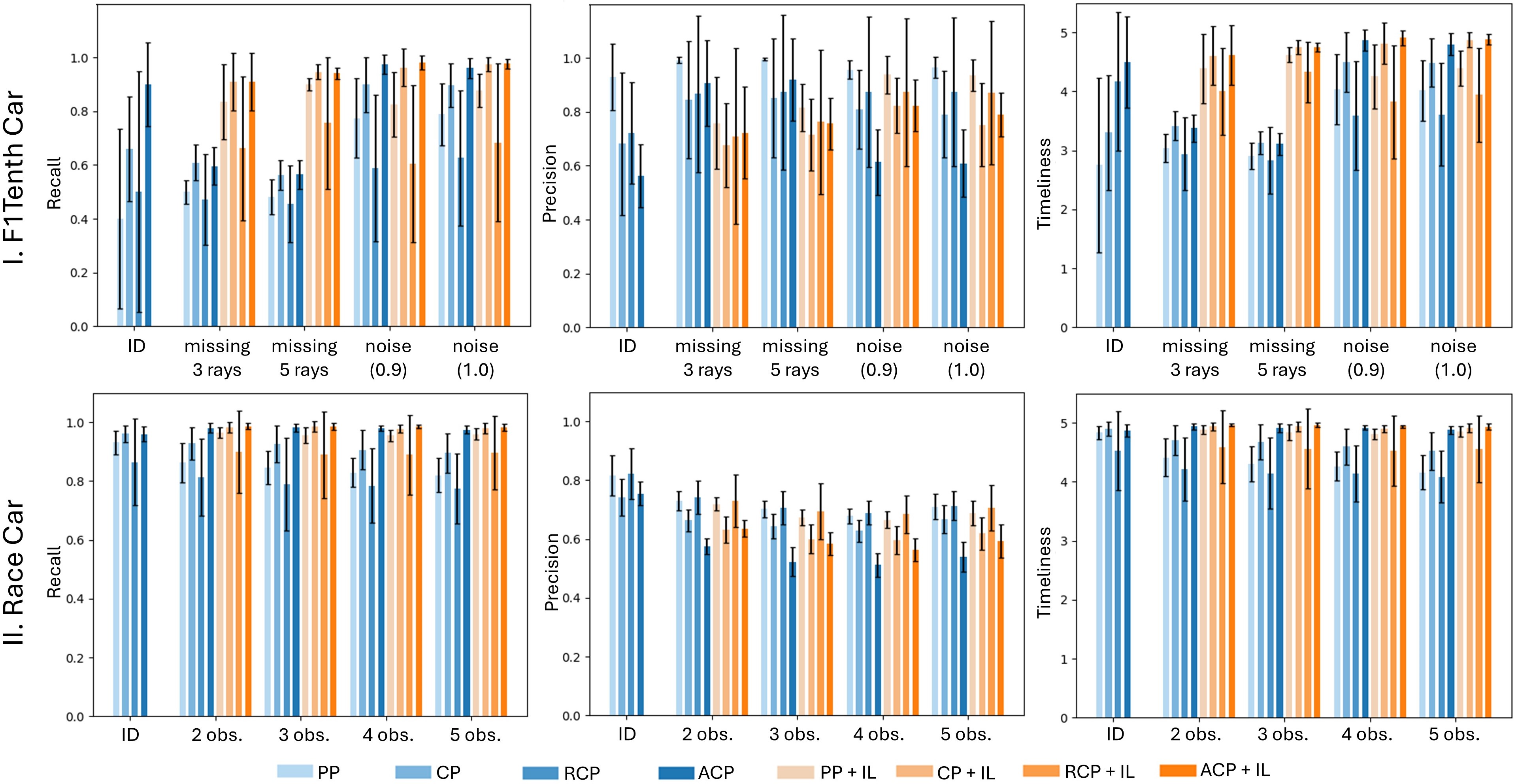}
    \caption{Recall, precision, and timeliness of our safety monitor for Case Studies I (top) and II (bottom), recorded over 10 trials. For Case Study I, the OOD scenarios are 3 missing LIDAR rays, 5 missing LIDAR rays, additive uniform (0,0.9) noise, and additive uniform (0,1.0) noise. For Case Study II, the OOD scenarios are 2, 3, 4, and 5 dynamic obstacles (obs.) on the race track.\label{fig:both_results}}
\end{figure*}

%% file: appendices/tbl_prediction_error.tex
\begin{table}[t]
\caption{Probability of a high-error robustness score prediction and the average displacement error (ADE) of our trajectory predictors with and without incremental learning.\label{tab:pred_err}}
\centering
     \begin{tabular}{|cccc|}
     \hline
     \multicolumn{4}{|c|}{\textbf{Case Study I}}\\
     \hline\hline
     \textbf{Setting} & \textbf{$\bm{P(\left|R_t\right| > \tau)}$} &\textbf{ADE w/o IL} & \textbf{ADE w/ IL}\\
     \hline
     ID & - & 0.052 $\pm$ 0.005 & - \\
     3 rays & 0.31 $\pm$ 0.06 & 0.090 $\pm$ 0.007 & 0.081 $\pm$ 0.007 \\
     5 rays & 0.30 $\pm$ 0.09 & 0.103 $\pm$ 0.008 & 0.081 $\pm$ 0.010 \\
     noise (0.9) & 0.43 $\pm$ 0.07 & 0.082 $\pm$ 0.005 & 0.045 $\pm$ 0.008 \\
     noise (1.0) & 0.44 $\pm$ 0.06 & 0.095 $\pm$ 0.005 & 0.055 $\pm$ 0.007 \\
     \hline
     \multicolumn{4}{c}{}\\
     \hline
     \multicolumn{4}{|c|}{\textbf{Case Study II}}\\
     \hline\hline
     \textbf{Setting} & \textbf{$\bm{P(\left|R_t\right| > \tau)}$} &\textbf{ADE w/o IL} & \textbf{ADE w/ IL}\\
     \hline
     ID & - & 0.332 $\pm$ 0.081 & - \\
     2 obs. & 0.76 $\pm$ 0.07 & 0.482 $\pm$ 0.080 & 0.266 $\pm$ 0.059 \\
     3 obs. & 0.78 $\pm$ 0.07 & 0.576 $\pm$ 0.090 & 0.300 $\pm$ 0.058 \\
     4 obs. & 0.80 $\pm$ 0.05 & 0.610 $\pm$ 0.096 & 0.297 $\pm$ 0.081 \\
     5 obs. & 0.84 $\pm$ 0.05 & 0.634 $\pm$ 0.092 & 0.268 $\pm$ 0.054 \\ 
     \hline
     \end{tabular}
\end{table}

%% file: appendices/ncs_dist.tex
\section{Non-Conformity Score Distributions} \label{app:ncs_dist}
Incremental learning (IL) recovers precision in all settings except when there are missing rays in Case Study I. In these settings, the distribution of the non-conformity scores before IL is thin-tailed, with a mean near the calibration mean (see Figure~\ref{fig:ncs_dist}). Thus, fewer samples exceed the high-error threshold $\tau$ for fine-tuning. Table~\ref{tab:pred_err} shows the estimated probability that the robustness score prediction will exceed $\tau$, along with the average displacement error of the predictors. In Case Study I with three or five missing rays, the probabilities of exceeding $\tau$ are the lowest. Hence, IL does not completely recover prediction quality to the in-distribution level. The ACP framework compensates for error with a more conservative prediction region over the robustness value, decreasing precision.

%% file: appendices/fig_il_example.tex
\begin{figure}[!t] 
    \centering
    \includegraphics[width=0.6\columnwidth]{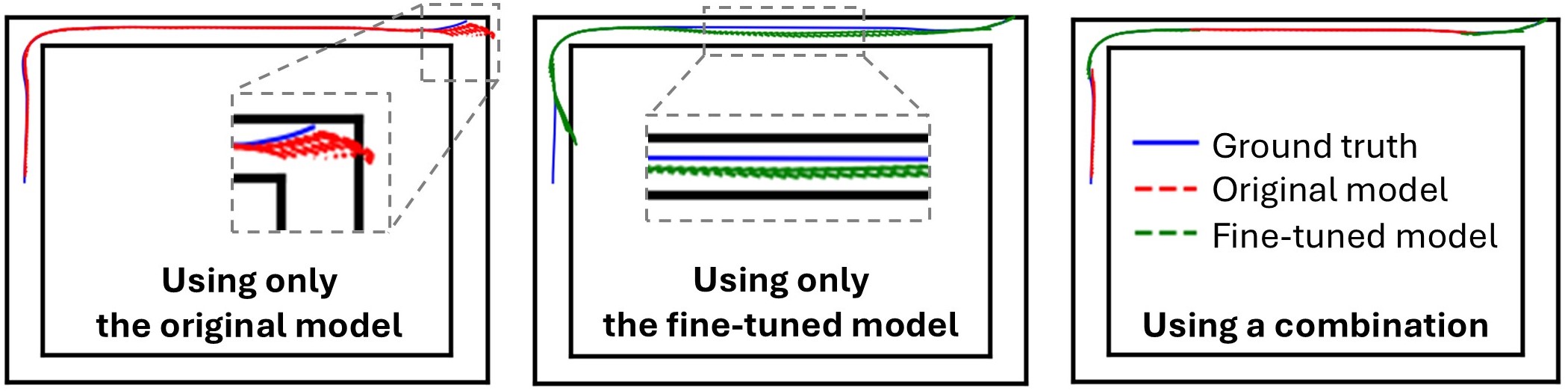}
    \caption{The original model makes poor predictions at the corners. We fine-tune our predictor on these high-error samples. The fine-tuned model learns to make higher quality predictions the corners, but forgets previously learned knowledge about the straight sections. To compensate, our method dynamically selects between the two models.
    \label{fig:f110_3_il}}
\end{figure}

%% file: appendices/fig_ncs_dist.tex
\begin{figure*}[!t] 
    \centering
    \includegraphics[width=\textwidth]{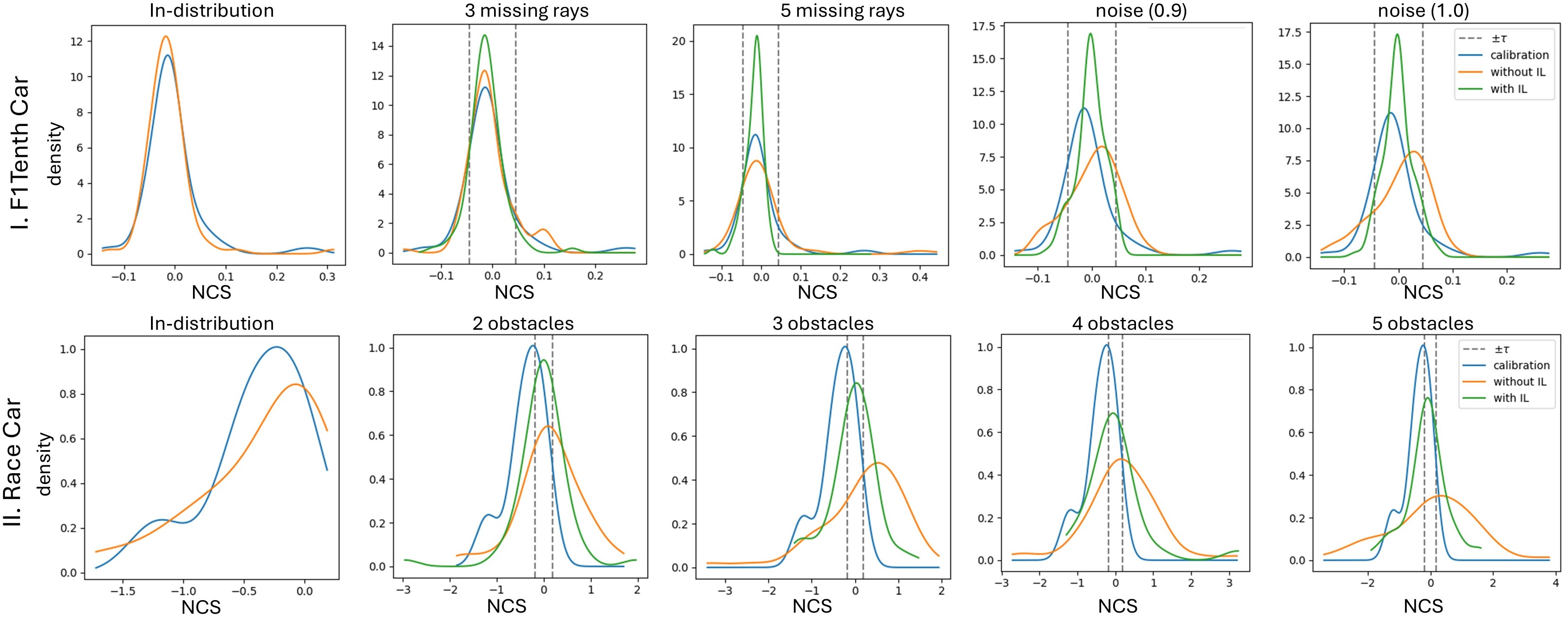}
    \caption{Distributions, estimated via Gaussian KDE, of the calibration non-conformity scores (NCS) and the online NCS with and without incremental learning (IL) for one seed. The high-error threshold $\tau$ for each OOD case is shown in dotted lines.
    \label{fig:ncs_dist}}
\end{figure*}

%% file: appendices/tbl_tvd.tex
\begin{table}[ht!]
\caption{Estimated total variation distance between the offline calibration and online non-conformity scores. In Case Study I, $\epsilon=0.08$, and in Case Study II, $\epsilon=0.03$. \label{tab:eps_est}}
\centering
\begin{minipage}[h]{0.47\linewidth}\centering
     \begin{tabular}{|ccc|}
     \hline
     \multicolumn{3}{|c|}{\textbf{Case Study I}}\\
     \hline\hline
     \textbf{Setting} & \textbf{w/o IL} & \textbf{w/ IL}\\ 
     \hline
     ID & 0.142 & - \\
     3 rays & 0.169 & 0.166\\
     5 rays & 0.209 & 0.212 \\
     noise (0.9) & 0.251 & 0.367 \\
     noise (1.0) & 0.243 & 0.321 \\ 
     \hline
     \end{tabular}
\end{minipage}
\begin{minipage}[h]{0.47\linewidth}\centering
     \begin{tabular}{|ccc|}
     \hline
     \multicolumn{3}{|c|}{\textbf{Case Study II}}\\
     \hline\hline
    \textbf{Setting} & \textbf{w/o IL} &\textbf{w/ IL}\\ 
     \hline
     ID & 0.159 & - \\
     2 obs. & 0.308 & 0.190 \\
     3 obs. & 0.354 & 0.212 \\
     4 obs. & 0.378 & 0.220 \\
     5 obs. & 0.454 & 0.222 \\ 
     \hline
 \end{tabular}
\end{minipage}
\end{table}

%% file: appendices/guarantees.tex
\section{Theoretical Guarantees} \label{app:emp_cov}
Neither conformal prediction (CP) nor robust conformal prediction (RCP) obtain theoretical guarantees in our case studies. The necessary assumptions are unsatisfied for both. For CP, the data is not exchangeable with the offline calibration set. For RCP, the distribution shift in non-conformity scores (NCS) exceeds the permissible amount $\epsilon$. Figure~\ref{fig:ncs_dist} shows the estimated distributions of these NCS. There is a large difference between the online and calibration distributions. Table~\ref{tab:eps_est} shows the total variation distance between these two distributions, estimated using the method in~\citet{zhao2024robust}. In all cases, the distance is much greater than $\epsilon$. In fact, since $\epsilon<\delta=0.1$ must hold, the assumptions required for RCP are not satisfied for any valid choice of $\epsilon$. Furthermore, even our \textit{in-distribution} samples exceed $\epsilon$, demonstrating
that it is in practice impossible to assume that even "in-distribution" data will always be within the $\epsilon$ bound. Table~\ref{tab:emp_cov_baselines} shows the empirical coverage for both techniques, which does not reach the target $1-\delta=0.9$ in nearly all cases, as expected.

%% file: appendices/tbl_empirical_coverage.tex
\begin{table}[ht!]
\caption{Empirical coverage for conformal prediction (CP) and robust conformal prediction (RCP) with and without incremental learning (IL).\label{tab:emp_cov_baselines}}
\centering
\begin{minipage}[h]{0.49\linewidth}\centering
    \setlength{\tabcolsep}{2.7pt}
     \begin{tabular}{|c|cc|cc|}
     \hline
     \multicolumn{5}{|c|}{\textbf{Case Study I}}\\
     \hline\hline
     \multirow{2}{*}{\textbf{Setting}} & \multirow{2}{*}{\textbf{CP}} & \textbf{CP} & \multirow{2}{*}{\textbf{RCP}} & \textbf{RCP}\\ 
     & & \textbf{+IL} & & \textbf{+IL}\\
     \hline
     ID & 0.92 & - & 0.41 & - \\
     3 rays & 0.89 & 0.91 & 0.39 & 0.41 \\
     5 rays & 0.88 & 0.91 & 0.37 & 0.41 \\
     noise (0.9) & 0.74 & 0.83 & 0.36 & 0.25 \\
     noise (1.0) & 0.72 & 0.87 & 0.37 & 0.26\\
     \hline
     \end{tabular}
\end{minipage}
\begin{minipage}[h]{0.49\linewidth}\centering
     \setlength{\tabcolsep}{2.8pt}
     \begin{tabular}{|c|cc|cc|}
     \hline
     \multicolumn{5}{|c|}{\textbf{Case Study II}}\\
     \hline\hline
    \multirow{2}{*}{\textbf{Setting}} & \multirow{2}{*}{\textbf{CP}} & \textbf{CP} & \multirow{2}{*}{\textbf{RCP}} & \textbf{RCP}\\ 
     & & \textbf{+IL} & & \textbf{+IL}\\
     \hline
     ID & 0.83 & - & 0.52 & - \\
     2 obs. & 0.72 & 0.85 & 0.43 & 0.50\\
     3 obs. & 0.64 & 0.82 & 0.38 & 0.48 \\
     4 obs. & 0.65 & 0.81 & 0.37 & 0.46 \\
     5 obs. & 0.64 & 0.83 & 0.39 & 0.51 \\
     \hline
 \end{tabular}
\end{minipage}
\end{table}

%% file: appendices/repeatability.tex
\section{Repeatability Package}\label{app:repeatability}
Our code\footnote{\href{https://doi.org/10.5281/zenodo.14835448}{https://doi.org/10.5281/zenodo.14835448}} reproduces Tables \ref{tab:recall}-\ref{tab:emp_cov_baselines} and Figures \ref{fig:acp_example}, \ref{fig:both_results}, and \ref{fig:ncs_dist}. Clone the Github repository and see the README for instructions to set up a Docker image\footnote{\href{https://hub.docker.com/r/vwlin/safety_monitoring}{https://hub.docker.com/r/vwlin/safety\_monitoring} (v1.0)} and run the code. Our results were obtained on a machine with Debian 11.0 and 96 CPU cores. To reproduce our results with pre-trained models, each case study requires approximately 10 hours. Our models were trained with CUDA 12.3 and 24 GB of GPU memory.

%% file: main.bbl
\begin{thebibliography}{50}
\providecommand{\natexlab}[1]{#1}
\providecommand{\url}[1]{\texttt{#1}}
\expandafter\ifx\csname urlstyle\endcsname\relax
  \providecommand{\doi}[1]{doi: #1}\else
  \providecommand{\doi}{doi: \begingroup \urlstyle{rm}\Url}\fi

\bibitem[Ahmed and Courville(2020)]{semantic}
Faruk Ahmed and Aaron Courville.
\newblock Detecting semantic anomalies.
\newblock In \emph{Proceedings of the AAAI Conference on Artificial Intelligence}, volume~34, pages 3154--3162, 2020.

\bibitem[Cai and Koutsoukos(2020)]{cai2020real}
Feiyang Cai and Xenofon Koutsoukos.
\newblock Real-time out-of-distribution detection in learning-enabled cyber-physical systems.
\newblock In \emph{2020 ACM/IEEE 11th International Conference on Cyber-Physical Systems (ICCPS)}, pages 174--183. IEEE, 2020.

\bibitem[Dixit et~al.(2023)Dixit, Lindemann, Wei, Cleaveland, Pappas, and Burdick]{lars_adaptive}
Anushri Dixit, Lars Lindemann, Skylar~X Wei, Matthew Cleaveland, George~J Pappas, and Joel~W Burdick.
\newblock Adaptive conformal prediction for motion planning among dynamic agents.
\newblock In \emph{Learning for Dynamics and Control Conference}, pages 300--314. PMLR, 2023.

\bibitem[Donz{\'e} and Maler(2010)]{stl}
Alexandre Donz{\'e} and Oded Maler.
\newblock Robust satisfaction of temporal logic over real-valued signals.
\newblock In \emph{International Conference on Formal Modeling and Analysis of Timed Systems}, pages 92--106. Springer, 2010.

\bibitem[Dreossi et~al.(2019)Dreossi, Donz{\'e}, and Seshia]{falsification}
Tommaso Dreossi, Alexandre Donz{\'e}, and Sanjit~A Seshia.
\newblock Compositional falsification of cyber-physical systems with machine learning components.
\newblock \emph{Journal of Automated Reasoning}, 63:\penalty0 1031--1053, 2019.

\bibitem[Feng et~al.(2021)Feng, Ng, and Easwaran]{arvind}
Yeli Feng, Daniel Jun~Xian Ng, and Arvind Easwaran.
\newblock Improving variational autoencoder based out-of-distribution detection for embedded real-time applications.
\newblock \emph{ACM Transactions on Embedded Computing Systems (TECS)}, 20\penalty0 (5s):\penalty0 1--26, 2021.

\bibitem[Gibbs and Candes(2021)]{acp}
Isaac Gibbs and Emmanuel Candes.
\newblock Adaptive conformal inference under distribution shift.
\newblock \emph{Advances in Neural Information Processing Systems}, 34:\penalty0 1660--1672, 2021.

\bibitem[He et~al.(2020)He, Mao, Shao, and Zhu]{he2020incremental}
Jiangpeng He, Runyu Mao, Zeman Shao, and Fengqing Zhu.
\newblock Incremental learning in online scenario.
\newblock In \emph{Proceedings of the IEEE/CVF conference on computer vision and pattern recognition}, pages 13926--13935, 2020.

\bibitem[Hendrycks and Dietterich(2019)]{hendrycks2019benchmarking}
Dan Hendrycks and Thomas Dietterich.
\newblock Benchmarking neural network robustness to common corruptions and perturbations.
\newblock \emph{arXiv preprint arXiv:1903.12261}, 2019.

\bibitem[Hendrycks and Gimpel(2016)]{hendrycks2016baseline}
Dan Hendrycks and Kevin Gimpel.
\newblock A baseline for detecting misclassified and out-of-distribution examples in neural networks.
\newblock \emph{arXiv preprint arXiv:1610.02136}, 2016.

\bibitem[Hendrycks et~al.(2019)Hendrycks, Mazeika, Kadavath, and Song]{aux}
Dan Hendrycks, Mantas Mazeika, Saurav Kadavath, and Dawn Song.
\newblock Using self-supervised learning can improve model robustness and uncertainty.
\newblock In \emph{Advances in Neural Information Processing Systems}, pages 15663--15674, 2019.

\bibitem[Ivanov et~al.(2020)Ivanov, Carpenter, Weimer, Alur, Pappas, and Lee]{ivanov2020case}
Radoslav Ivanov, Taylor~J Carpenter, James Weimer, Rajeev Alur, George~J Pappas, and Insup Lee.
\newblock Case study: verifying the safety of an autonomous racing car with a neural network controller.
\newblock In \emph{Proceedings of the 23rd International Conference on Hybrid Systems: Computation and Control}, pages 1--7, 2020.

\bibitem[Kantaros et~al.(2021)Kantaros, Carpenter, Sridhar, Yang, Lee, and Weimer]{k2021real}
Yiannis Kantaros, Taylor Carpenter, Kaustubh Sridhar, Yahan Yang, Insup Lee, and James Weimer.
\newblock Real-time detectors for digital and physical adversarial inputs to perception systems.
\newblock In \emph{Proceedings of the ACM/IEEE 12th International Conference on Cyber-Physical Systems}, pages 67--76, 2021.

\bibitem[Kaur et~al.(2021)Kaur, Jha, Roy, Park, Sokolsky, and Lee]{kaur}
Ramneet Kaur, Susmit Jha, Anirban Roy, Sangdon Park, Oleg Sokolsky, and Insup Lee.
\newblock Detecting oods as datapoints with high uncertainty.
\newblock \emph{arXiv preprint arXiv:2108.06380}, 2021.

\bibitem[Kaur et~al.(2022)Kaur, Jha, Roy, Park, Dobriban, Sokolsky, and Lee]{iDECODe}
Ramneet Kaur, Susmit Jha, Anirban Roy, Sangdon Park, Edgar Dobriban, Oleg Sokolsky, and Insup Lee.
\newblock i{DECOD}e: In-distribution {E}quivariance for {C}onformal {O}ut-of-distribution {D}etection, {A}ssociation for the {A}dvancement of {A}rtificial {I}ntelligence, 2022.

\bibitem[Kaur et~al.(2023{\natexlab{a}})Kaur, Ji, Dutta, Caprio, Yang, Bernardis, Sokolsky, and Lee]{semantic_ood_detection}
Ramneet Kaur, Xiayan Ji, Souradeep Dutta, Michele Caprio, Yahan Yang, Elena Bernardis, Oleg Sokolsky, and Insup Lee.
\newblock Using semantic information for defining and detecting ood inputs.
\newblock \emph{arXiv preprint arXiv:2302.11019}, 2023{\natexlab{a}}.

\bibitem[Kaur et~al.(2023{\natexlab{b}})Kaur, Kantaros, Si, Weimer, and Lee]{time_series_adv}
Ramneet Kaur, Yiannis Kantaros, Wenwen Si, James Weimer, and Insup Lee.
\newblock Detection of adversarial physical attacks in time-series image data.
\newblock \emph{arXiv preprint arXiv:2304.13919}, 2023{\natexlab{b}}.

\bibitem[Kaur et~al.(2023{\natexlab{c}})Kaur, Sridhar, Park, Yang, Jha, Roy, Sokolsky, and Lee]{codit}
Ramneet Kaur, Kaustubh Sridhar, Sangdon Park, Yahan Yang, Susmit Jha, Anirban Roy, Oleg Sokolsky, and Insup Lee.
\newblock {COD}i{T}: Conformal {O}ut-of-{D}istribution {D}etection in {T}ime-{S}eries {D}ata.
\newblock In \emph{Proceedings of the ACM/IEEE 14th International Conference on Cyber-Physical Systems (with CPS-IoT Week 2023)}, pages 120--131, 2023{\natexlab{c}}.

\bibitem[Kaur et~al.(2024)Kaur, Yang, Sokolsky, and Lee]{kaur2024out}
Ramneet Kaur, Yahan Yang, Oleg Sokolsky, and Insup Lee.
\newblock Out-of-distribution detection in dependent data for cyber-physical systems with conformal guarantees.
\newblock \emph{ACM Transactions on Cyber-Physical Systems}, 2024.

\bibitem[Lample and Chaplot(2017)]{atari}
Guillaume Lample and Devendra~Singh Chaplot.
\newblock Playing fps games with deep reinforcement learning.
\newblock In \emph{Proceedings of the AAAI Conference on Artificial Intelligence}, volume~31, 2017.

\bibitem[Leurent(2018)]{highway-env}
Edouard Leurent.
\newblock An environment for autonomous driving decision-making.
\newblock \url{https://github.com/eleurent/highway-env}, 2018.

\bibitem[Lindemann et~al.(2023)Lindemann, Qin, Deshmukh, and Pappas]{lars_iccps}
Lars Lindemann, Xin Qin, Jyotirmoy~V Deshmukh, and George~J Pappas.
\newblock Conformal prediction for stl runtime verification.
\newblock In \emph{Proceedings of the ACM/IEEE 14th International Conference on Cyber-Physical Systems (with CPS-IoT Week 2023)}, pages 142--153, 2023.

\bibitem[M et~al.(2022)M, Khullar, Bhosle, Salunke, Bangare, and Ingavale]{cps-incremental}
Veeramanickam~M.R. M, Vikas Khullar, Amol~A Bhosle, Mangesh~D. Salunke, Jyoti~L. Bangare, and Aniket Ingavale.
\newblock Streamed incremental learning for cyber attack classification using machine learning.
\newblock In \emph{2022 2nd International Conference on Innovative Sustainable Computational Technologies (CISCT)}, pages 1--5, 2022.
\newblock \doi{10.1109/CISCT55310.2022.10046651}.

\bibitem[Mac{\^e}do et~al.(2021)Mac{\^e}do, Ren, Zanchettin, Oliveira, and Ludermir]{unsuper_2}
David Mac{\^e}do, Tsang~Ing Ren, Cleber Zanchettin, Adriano~LI Oliveira, and Teresa Ludermir.
\newblock Entropic out-of-distribution detection.
\newblock In \emph{2021 International Joint Conference on Neural Networks (IJCNN)}, pages 1--8. IEEE, 2021.

\bibitem[Maler and Nickovic(2004)]{maler2004monitoring}
Oded Maler and Dejan Nickovic.
\newblock Monitoring temporal properties of continuous signals.
\newblock In \emph{International symposium on formal techniques in real-time and fault-tolerant systems}, pages 152--166. Springer, 2004.

\bibitem[McCloskey and Cohen(1989)]{mccloskey1989catastrophic}
Michael McCloskey and Neal~J Cohen.
\newblock Catastrophic interference in connectionist networks: The sequential learning problem.
\newblock In \emph{Psychology of learning and motivation}, volume~24, pages 109--165. Elsevier, 1989.

\bibitem[Mohammed and Valdenegro-Toro(2021)]{mohammed2021benchmark}
Aaqib~Parvez Mohammed and Matias Valdenegro-Toro.
\newblock Benchmark for out-of-distribution detection in deep reinforcement learning.
\newblock \emph{arXiv preprint arXiv:2112.02694}, 2021.

\bibitem[Papadopoulos et~al.(2002)Papadopoulos, Proedrou, Vovk, and Gammerman]{papadopoulos2002inductive}
Harris Papadopoulos, Kostas Proedrou, Volodya Vovk, and Alex Gammerman.
\newblock Inductive confidence machines for regression.
\newblock In \emph{Machine learning: ECML 2002: 13th European conference on machine learning Helsinki, Finland, August 19--23, 2002 proceedings 13}, pages 345--356. Springer, 2002.

\bibitem[Parikh and Polikar(2007)]{ensemble-incremental}
Devi Parikh and Robi Polikar.
\newblock An ensemble-based incremental learning approach to data fusion.
\newblock \emph{IEEE Transactions on Systems, Man, and Cybernetics, Part B (Cybernetics)}, 37\penalty0 (2):\penalty0 437--450, 2007.
\newblock \doi{10.1109/TSMCB.2006.883873}.

\bibitem[Ramakrishna et~al.(2022)Ramakrishna, Rahiminasab, Karsai, Easwaran, and Dubey]{ramakrishna2022efficient}
Shreyas Ramakrishna, Zahra Rahiminasab, Gabor Karsai, Arvind Easwaran, and Abhishek Dubey.
\newblock Efficient out-of-distribution detection using latent space of $\beta$-vae for cyber-physical systems.
\newblock \emph{ACM Transactions on Cyber-Physical Systems (TCPS)}, 6\penalty0 (2):\penalty0 1--34, 2022.

\bibitem[Rebuffi et~al.(2017)Rebuffi, Kolesnikov, Sperl, and Lampert]{rebuffi2017icarl}
Sylvestre-Alvise Rebuffi, Alexander Kolesnikov, Georg Sperl, and Christoph~H Lampert.
\newblock icarl: Incremental classifier and representation learning.
\newblock In \emph{Proceedings of the IEEE conference on Computer Vision and Pattern Recognition}, pages 2001--2010, 2017.

\bibitem[Recht et~al.(2019)Recht, Roelofs, Schmidt, and Shankar]{recht2019imagenet}
Benjamin Recht, Rebecca Roelofs, Ludwig Schmidt, and Vaishaal Shankar.
\newblock Do imagenet classifiers generalize to imagenet?
\newblock In \emph{International conference on machine learning}, pages 5389--5400. PMLR, 2019.

\bibitem[Reis et~al.(2020)Reis, Murillo~Piedrahita, Rueda, Fernandes, Medeiros, de~Amorim, and Mattos]{reis2020unsupervised-incremental}
L{\'u}cio Henrik~A Reis, Andres Murillo~Piedrahita, Sandra Rueda, Nat{\'a}lia~C Fernandes, Dianne~SV Medeiros, Marcelo~Dias de~Amorim, and Diogo~MF Mattos.
\newblock Unsupervised and incremental learning orchestration for cyber-physical security.
\newblock \emph{Transactions on emerging telecommunications technologies}, 31\penalty0 (7):\penalty0 e4011, 2020.

\bibitem[Saunders et~al.(1999)Saunders, Gammerman, and Vovk]{saunders1999transduction}
Craig Saunders, Alex Gammerman, and Volodya Vovk.
\newblock Transduction with confidence and credibility.
\newblock 1999.

\bibitem[Siddiqui()]{tesla}
Faiz Siddiqui.
\newblock Tesla is putting 'self-driving' in the hands of drivers amid criticism the tech is not ready.
\newblock \emph{The Washington Post}.
\newblock URL \url{https://www.washingtonpost.com/technology/2020/10/21/tesla-self-driving/}.

\bibitem[Sridhar et~al.(2022)Sridhar, Dutta, Kaur, Weimer, Sokolsky, and Lee]{sridhar2022towards}
Kaustubh Sridhar, Souradeep Dutta, Ramneet Kaur, James Weimer, Oleg Sokolsky, and Insup Lee.
\newblock Towards alternative techniques for improving adversarial robustness: Anaflysis of adversarial training at a spectrum of perturbations.
\newblock \emph{arXiv preprint arXiv:2206.06496}, 2022.

\bibitem[Sundar et~al.(2020)Sundar, Ramakrishna, Rahiminasab, Easwaran, and Dubey]{beta-vae-2}
Vijaya~Kumar Sundar, Shreyas Ramakrishna, Zahra Rahiminasab, Arvind Easwaran, and Abhishek Dubey.
\newblock Out-of-distribution detection in multi-label datasets using latent space of $\beta$-vae.
\newblock In \emph{2020 IEEE Security and Privacy Workshops (SPW)}, pages 250--255. IEEE, 2020.

\bibitem[Tack et~al.(2020)Tack, Mo, Jeong, and Shin]{csi}
Jihoon Tack, Sangwoo Mo, Jongheon Jeong, and Jinwoo Shin.
\newblock Csi: Novelty detection via contrastive learning on distributionally shifted instances.
\newblock \emph{Advances in Neural Information Processing Systems}, 33, 2020.

\bibitem[Tan()]{waymo}
Eli Tan.
\newblock Waymo's robot taxis are almost mainstream. can they now turn a profit?
\newblock \emph{The New York Times}.
\newblock URL \url{https://www.nytimes.com/2024/09/04/technology/waymo-expansion-alphabet.html}.

\bibitem[Taylor et~al.(2022)Taylor, Kardas, Cucurull, Scialom, Hartshorn, Saravia, Poulton, Kerkez, and Stojnic]{galactica}
Ross Taylor, Marcin Kardas, Guillem Cucurull, Thomas Scialom, Anthony Hartshorn, Elvis Saravia, Andrew Poulton, Viktor Kerkez, and Robert Stojnic.
\newblock Galactica: A large language model for science.
\newblock \emph{arXiv preprint arXiv:2211.09085}, 2022.

\bibitem[Tiwari et~al.(2014)Tiwari, Dutertre, Jovanovi{\'c}, de~Candia, Lincoln, Rushby, Sadigh, and Seshia]{tiwari2014safety}
Ashish Tiwari, Bruno Dutertre, Dejan Jovanovi{\'c}, Thomas de~Candia, Patrick~D Lincoln, John Rushby, Dorsa Sadigh, and Sanjit Seshia.
\newblock Safety envelope for security.
\newblock In \emph{Proceedings of the 3rd international conference on High confidence networked systems}, pages 85--94, 2014.

\bibitem[Vishwakarma et~al.(2024)Vishwakarma, Lin, and Vinayak]{vishwakarma2024taming}
Harit Vishwakarma, Heguang Lin, and Ramya~Korlakai Vinayak.
\newblock Taming false positives in out-of-distribution detection with human feedback.
\newblock \emph{arXiv preprint arXiv:2404.16954}, 2024.

\bibitem[Vovk et~al.(1999)Vovk, Gammerman, and Saunders]{vovk1999machine}
Volodya Vovk, Alexander Gammerman, and Craig Saunders.
\newblock Machine-learning applications of algorithmic randomness.
\newblock 1999.

\bibitem[Xiao et~al.(2014)Xiao, Zhang, Yang, Peng, and Zhang]{incremental_learning_image}
Tianjun Xiao, Jiaxing Zhang, Kuiyuan Yang, Yuxin Peng, and Zheng Zhang.
\newblock Error-driven incremental learning in deep convolutional neural network for large-scale image classification.
\newblock In \emph{Proceedings of the 22nd ACM International Conference on Multimedia}, MM '14, page 177–186, New York, NY, USA, 2014. Association for Computing Machinery.
\newblock ISBN 9781450330633.
\newblock \doi{10.1145/2647868.2654926}.
\newblock URL \url{https://doi.org/10.1145/2647868.2654926}.

\bibitem[Yang et~al.(2022)Yang, Kaur, Dutta, and Lee]{interpretable_ood}
Yahan Yang, Ramneet Kaur, Souradeep Dutta, and Insup Lee.
\newblock Interpretable detection of distribution shifts in learning enabled cyber-physical systems.
\newblock In \emph{2022 ACM/IEEE 13th International Conference on Cyber-Physical Systems (ICCPS)}, pages 225--235. IEEE, 2022.

\bibitem[Yang et~al.(2023)Yang, Dutta, Jang, Sokolsky, and Lee]{yang2023incremental}
Yahan Yang, Souradeep Dutta, Kuk~Jin Jang, Oleg Sokolsky, and Insup Lee.
\newblock Incremental learning with memory regressors for motion prediction in autonomous racing.
\newblock In \emph{Proceedings of the ACM/IEEE 14th International Conference on Cyber-Physical Systems (with CPS-IoT Week 2023)}, pages 264--265, 2023.

\bibitem[Yang et~al.(2024)Yang, Kaur, Dutta, and Lee]{yang2024memory}
Yahan Yang, Ramneet Kaur, Souradeep Dutta, and Insup Lee.
\newblock Memory-based distribution shift detection for learning enabled cyber-physical systems with statistical guarantees.
\newblock \emph{ACM Transactions on Cyber-Physical Systems}, 8\penalty0 (2):\penalty0 1--28, 2024.

\bibitem[Yuan et~al.(2021)Yuan, Weng, Ou, and Kitani]{agentformer}
Ye~Yuan, Xinshuo Weng, Yanglan Ou, and Kris~M Kitani.
\newblock Agentformer: Agent-aware transformers for socio-temporal multi-agent forecasting.
\newblock In \emph{Proceedings of the IEEE/CVF International Conference on Computer Vision}, pages 9813--9823, 2021.

\bibitem[Zhao et~al.(2024)Zhao, Hoxha, Fainekos, Deshmukh, and Lindemann]{zhao2024robust}
Yiqi Zhao, Bardh Hoxha, Georgios Fainekos, Jyotirmoy~V Deshmukh, and Lars Lindemann.
\newblock Robust conformal prediction for stl runtime verification under distribution shift.
\newblock In \emph{2024 ACM/IEEE 15th International Conference on Cyber-Physical Systems (ICCPS)}, pages 169--179. IEEE, 2024.

\bibitem[Zisselman and Tamar(2020)]{super_2}
Ev~Zisselman and Aviv Tamar.
\newblock Deep residual flow for out of distribution detection.
\newblock In \emph{Proceedings of the IEEE/CVF Conference on Computer Vision and Pattern Recognition}, pages 13994--14003, 2020.

\end{thebibliography}
